\documentclass[11pt,reqno,twoside]{article}

\usepackage{fixltx2e} %

\usepackage{cmap} %

\usepackage[T1]{fontenc}
\usepackage[utf8]{inputenc}
\usepackage{graphicx}
\usepackage{placeins}
\usepackage{enumerate}

\usepackage{subcaption}

\usepackage{verbatim}

\usepackage{floatrow}
\newfloatcommand{capbtabbox}{table}[][\FBwidth]

\usepackage{blindtext}

\usepackage{setspace}

\let\counterwithin\relax  %
\usepackage{lmodern} %
\usepackage[scale=0.80]{tgheros} %

\usepackage{bm} %

\usepackage{amsmath,amsbsy,amsgen,amscd,amsthm,amsfonts,amssymb} 

\usepackage[centering,top=1.5in,bottom=1.5in,left=1in,right=1in]{geometry}

\usepackage{titling}
\usepackage{cases}

\graphicspath{ {./images/} }

\usepackage[sf,bf,compact]{titlesec}

\usepackage{booktabs,longtable,tabu} %
\usepackage[font=small,margin=12pt,labelfont={sf,bf},labelsep={space}]{caption}

\usepackage[usenames,dvipsnames]{xcolor}

\definecolor{dark-gray}{gray}{0.3}
\definecolor{dkgray}{rgb}{.4,.4,.4}
\definecolor{dkblue}{rgb}{0,0,.5}
\definecolor{medblue}{rgb}{0,0,.75}
\definecolor{rust}{rgb}{0.5,0.1,0.1}

\usepackage{url}
\usepackage[colorlinks=true]{hyperref}
\hypersetup{linkcolor=dkblue}    
\hypersetup{citecolor=rust}      
\hypersetup{urlcolor=rust}     

\usepackage[final]{microtype}

\newtheoremstyle{myThm} %
    {\topsep}                    %
    {\topsep}                    %
    {\itshape}                   %
    {}                           %
    {\sffamily\bfseries}                   %
    {.}                          %
    {.5em}                       %
    {}  %

\newtheoremstyle{myRem} %
    {\topsep}                    %
    {\topsep}                    %
    {}                   %
    {}                           %
    {\sffamily}                   %
    {.}                          %
    {.5em}                       %
    {}  %

\newtheoremstyle{myDef} %
    {\topsep}                    %
    {\topsep}                    %
    {}                   %
    {}                           %
    {\sffamily\bfseries}                   %
    {.}                          %
    {.5em}                       %
    {}  %

\theoremstyle{myThm}
\newtheorem{theorem}{Theorem}[section]

\newtheorem{proposition}[theorem]{Proposition}

\theoremstyle{myRem}
\newtheorem{remark}[theorem]{Remark}

\theoremstyle{myDef}

\let\originalleft\left
\let\originalright\right
\renewcommand{\left}{\mathopen{}\mathclose\bgroup\originalleft}
\renewcommand{\right}{\aftergroup\egroup\originalright}

\usepackage{mathtools}
\mathtoolsset{centercolon}  %

\definecolor{mygreen}{rgb}{0.1,0.75,0.2}

\newcommand{\nc}{\normalcolor}

\usepackage[]{algorithm}
\usepackage{algorithmic}

\usepackage{graphicx}

\usepackage{soul}
\usepackage{authblk}
\usepackage[square,numbers]{natbib}
\usepackage{chngcntr}
\usepackage{mathrsfs}
\counterwithin{table}{section}
\counterwithin{figure}{section}

\counterwithin{algorithm}{section}

\title{Enhancing Gaussian Process Surrogates for Optimization and Posterior Approximation via Random Exploration}   %

\author{Hwanwoo Kim and Daniel Sanz-Alonso}
\vspace{.25in} 
\date{University of Chicago}

\makeatletter\@addtoreset{section}{part}\makeatother%
\numberwithin{equation}{section}

\newcommand{\upperRomannumeral}[1]{\uppercase\expandafter{\romannumeral#1}}

\usepackage{mathtools}

\begin{document}
\maketitle

\begin{abstract}
This paper proposes novel noise-free Bayesian optimization strategies that rely on a random exploration step to enhance the accuracy of Gaussian process surrogate models. The new algorithms retain the ease of implementation of the classical GP-UCB algorithm, but the additional random exploration step accelerates their convergence, nearly achieving the optimal convergence rate. Furthermore, to facilitate Bayesian inference with an intractable likelihood, we propose to utilize
optimization iterates for \emph{maximum a posteriori} estimation
to build a Gaussian process surrogate model for the unnormalized log-posterior density. We provide bounds for the Hellinger distance between the true and the approximate posterior distributions in terms of the number of design points.
We demonstrate the effectiveness of our Bayesian optimization algorithms in non-convex benchmark objective functions, in a machine learning hyperparameter tuning problem, and in a black-box engineering design problem. The effectiveness of our posterior approximation approach is demonstrated in two Bayesian inference problems for parameters of dynamical systems. 
\end{abstract}

\section{Introduction}
Gaussian processes are a powerful, non-parametric tool widely used in various fields of computational mathematics, machine learning, and statistics \cite{wendland2004scattered, williams2006gaussian, stein2012interpolation}. 
They provide a probabilistic approach to modeling complex functions in optimization and posterior approximation tasks \cite{frazier2018tutorial,stuart2018posterior}. 
In optimization,  Gaussian processes are employed to develop Bayesian optimization algorithms, which are particularly effective for global optimization of black-box or expensive-to-evaluate objective functions \cite{jones1998efficient, frazier2018tutorial}. 
In posterior approximation, Gaussian processes serve as surrogate models to approximate the posterior distribution
when dealing with intractable likelihoods \cite{ stuart2018posterior, cleary2021calibrate}. By offering a flexible probabilistic framework, Gaussian processes facilitate the development of sophisticated algorithms that can handle a wide range of challenging problems in scientific computing, engineering, machine learning, and beyond.

Bayesian optimization algorithms \cite{jones1998efficient, mockus1998application,frazier2018tutorial} sequentially acquire information on the objective by observing its value at carefully selected query points, balancing exploration and exploitation by leveraging uncertainty estimates provided by a Gaussian process surrogate model of the objective. In some applications, the observations of the objective are noisy, but in many others the objective can be noiselessly observed; examples include hyperparameter tuning for machine learning algorithms \cite{burges1996improving}, parameter estimation for computer models \cite{clark2016engineering, pourmohamad2020compmodels}, goal-driven dynamics learning \cite{bansal2017goal}, and alignment of density maps in Wasserstein distance \cite{singer2023alignment}. 
While most Bayesian optimization algorithms can be implemented with either noisy or noise-free observations, few methods and theoretical analyses are tailored to the noise-free setting. 

This paper introduces two new algorithms for Bayesian optimization with noise-free observations. The first algorithm, which we call GP-UCB+, supplements query points obtained via the classical GP-UCB algorithm \cite{srinivas2009gaussian} with randomly sampled query points. The second algorithm, which we call EXPLOIT+, supplements query points obtained by maximizing the posterior mean of a Gaussian process surrogate model with randomly sampled query points. Both algorithms retain the simplicity and ease of implementation of the GP-UCB algorithm, but introduce an additional random exploration step to ensure that the fill-distance of query points decays at a near-optimal rate, thus enhancing the accuracy of surrogate models for the objective function. The new random exploration step has a profound impact on both theoretical guarantees and empirical performance. On the one hand, the convergence rate of GP-UCB+ and EXPLOIT+ improves upon existing and refined rates for the GP-UCB algorithm. Indeed, the new algorithms nearly achieve the optimal convergence rate in \cite{bull2011convergence}. On the other hand, GP-UCB+ and EXPLOIT+ explore the state space faster, which leads to an improvement in numerical performance across a range of benchmark and real-world examples. 

Finally, this paper explores using Bayesian optimization iterates for \emph{maximum a posteriori} (MAP) estimation as an experimental design tool to build a Gaussian process surrogate model for the unnormalized log-posterior density. The proposed optimization algorithms then provide a systematic way to build a surrogate model that accurately reflects the \emph{local} behavior of the true posterior around its mode, without sacrificing the depiction of its \emph{global} behavior. The effectiveness of these strategies is theoretically supported by a convergence rate analysis and numerically supported by several examples where the likelihood function involves differential equations and cannot be evaluated in closed form.

\subsection{Main Contributions}
\begin{itemize}
    \item We introduce two new algorithms, GP-UCB+ and EXPLOIT+, whose convergence rate (simple regret bound) established in Theorem \ref{thm:CumulativeRegretBounds} nearly matches the optimal convergence rate in \cite{bull2011convergence}. The proposed algorithms far improve existing and refined rates for the classical GP-UCB algorithm and its variants. En route to studying the convergence rate for our new algorithms, we establish in Theorem \ref{UCB_cum_reg_bd} a cumulative regret bound for GP-UCB with squared exponential kernels that refines the one in \cite{lyu2019efficient}. 
    \item  We numerically demonstrate that GP-UCB+ and EXPLOIT+ outperform GP-UCB and other popular Bayesian optimization algorithms across many examples, including optimization of several $10$-dimensional benchmark objective functions, hyperparameter tuning for random forests, and optimal parameter estimation of a garden sprinkler computer model.  
    \item We showcase that both GP-UCB+ and EXPLOIT+ share the simplicity and ease of implementation of the GP-UCB algorithm. In addition, EXPLOIT+ requires fewer input parameters than GP-UCB or GP-UCB+, and achieves competitive empirical performance without any tuning.
    \item In Theorem  \ref{thm:pos_approx_log_unnorm},  we bound the Hellinger distance between the true posterior and surrogate posteriors that use design points obtained with GP-UCB+ and EXPLOIT+ algorithms. We numerically demonstrate the effectiveness of GP-UCB+ and EXPLOIT+ algorithms in building Gaussian process surrogates to facilitate approximate Bayesian inference for parameters of differential equations. 
\end{itemize}

\subsection{Outline}
Section \ref{sec:Background} formalizes the optimization problem of interest and provides necessary background. We review related work in
Section \ref{sec:relatedwork}. Our new Bayesian optimization algorithms are introduced in Section \ref{Sec:Decouple}, where we establish regret bounds under a deterministic assumption on the objective function.
Section \ref{sec:postapproxandsampling} utilizes iterates from the proposed optimization strategies as design points to build a Gaussian process surrogate model and facilitate computationally efficient Bayesian inference with intractable or expensive-to-evaluate likelihood functions. 
Section \ref{sec:numerics} contains numerical examples, and we close in Section \ref{sec:conclusions}.
Proofs and additional numerical experiments are deferred to an appendix.

\section{Preliminaries}%
\label{sec:Background}
\subsection{Problem Statement}
We want to find the global maximizer of an objective function $f:\mathcal{X}\to \mathbb{R}$ by leveraging the observed values of $f$ at carefully chosen query points. We are interested in the setting where the observations of the objective are noise-free, i.e. for query points $X_t = \{x_1, \ldots, x_t\}$ we can access observations $F_t = [f(x_1), \ldots, f(x_t)]^\top.$ The functional form of $f$ is not assumed to be known. For simplicity, we assume throughout that $\mathcal{X}\subset \mathbb{R}^d$ is a $d$-dimensional hypercube. We assume that $f \in \mathcal{H}_{k}(\mathcal{X})$ belongs to the Reproducing Kernel Hilbert Space (RKHS) associated with a kernel $k: \mathcal{X}\times \mathcal{X} \to \mathbb{R}$.

\subsection{Gaussian Processes and Bayesian Optimization}
\label{subsec: GP}
Many Bayesian optimization algorithms, including the ones introduced in this paper, rely on a  Gaussian process surrogate model of the objective function to guide the choice of query points. Here, we review the main ideas.
Denote generic query locations by $X_t = \{x_1, \ldots, x_t\}$ and the corresponding noise-free observations by $F_t = [f(x_1), \ldots, f(x_t)]^\top.$
 Gaussian process interpolation with a prior $\mathcal{GP}(0, k)$ yields the following posterior predictive mean and variance:
\begin{align*}
   \mu_{t,0}(x) &= k_t(x)^\top K_{tt}^{-1} F_t, \\
   \sigma_{t,0}^2(x) &= k(x,x) - k_t(x)^\top K_{tt}^{-1} k_t(x),
\end{align*}
where $k_t(x) = [k(x, x_1), \ldots, k(x, x_t)]^\top$ and $K_{tt}$ is a $t\times t$ matrix with entries $(K_{tt})_{i,j} = k(x_i, x_j)$.

Our interest lies in Bayesian optimization with noise-free observations. However, we recall for later reference that if the observations are noisy and take the form $y_i = f(x_i) + \eta_i,$  $1 \le i \le t,$  where $\eta_i \stackrel{\text{i.i.d.}}{\sim} N(0,\lambda),$ then the posterior predictive mean and variance are given by
\begin{align*}
   \mu_{t,\lambda}(x) &= k_t(x)^\top \left(K_{tt} + \lambda I\right)^{-1} Y_t, \\
   \sigma_{t,\lambda}^2(x) &= k(x,x) - k_t(x)^\top (K_{tt}+\lambda I)^{-1} k_t(x),
\end{align*}
where $Y_t = [y_1, \ldots, y_t]^\top.$

To perform Bayesian optimization, one can sequentially select query points by optimizing a Gaussian Process Upper Confidence Bound (GP-UCB) acquisition function. Let $X_{t-1} = \{x_1, \ldots, x_{t-1}\}$ denote the query points at the $(t-1)$-th iteration of the algorithm.
Then, at the $t$-th iteration, the classical GP-UCB algorithm \cite{srinivas2009gaussian} sets
\begin{align}\label{eq:GP-UCB}
    x_{t} = \arg \max_{x \in \mathcal{X}} \mu_{t-1,\lambda}(x) + \beta_{t}^{\frac12} \sigma_{t-1,\lambda}(x),
\end{align}
where $\beta_{t}$ is a user-chosen positive parameter. The posterior predictive mean provides a surrogate model for the objective; hence, one expects the maximum of $f$ to be achieved at a point $x\in \mathcal{X}$ where $\mu_{t-1,\lambda}(x)$ is large. However, the surrogate model $\mu_{t-1,\lambda}(x)$ may not be accurate at points $x\in \mathcal{X}$ where $\sigma_{t-1,\lambda}^2(x)$ is large, and selecting query points with large predictive variance helps improve the accuracy of the surrogate model. The GP-UCB algorithm finds a compromise between exploitation (maximizing the mean) and exploration (maximizing the variance). The weight parameter $\beta_t$ balances this exploitation-exploration trade-off. For later discussion, Algorithm \ref{noiseless_BO} below summarizes the approach with noise-free observations $F_t$ and $\lambda =0.$

\begin{algorithm}[htp]
\caption{GP-UCB with noise-free observations. \label{noiseless_BO}} 
\begin{algorithmic}[1]
    \STATE {\bf Input}: Kernel $k;$ Total number of iterations $T;$ Initial design points $X_0;$ Initial noise-free observations $F_0$; Weights $\{\beta_t\}_{t=1}^{T}.$
    \STATE Construct $\mu_{0,0}(x)$ and $\sigma_{0,0}(x)$ using $X_0$ and $F_0$.
    \STATE {\bf For} $t = 1, \ldots, T$ {\bf do}: 
    \begin{enumerate}
    \item \vspace{-0.25cm} Set
    \setlength{\belowdisplayskip}{10pt} \setlength{\belowdisplayshortskip}{5pt}
\setlength{\abovedisplayskip}{0pt} \setlength{\abovedisplayshortskip}{0pt}
\[ x_t = \arg\max_{x \in \mathcal{X}} \mu_{t-1,0}(x) +\beta_{t}^{\frac12} \sigma_{t-1,0}(x).\]
    \item \vspace{-0.35cm} Set $X_t = X_{t-1} \cup \{
x_t\}$ and $F_t = F_{t-1} \cup \{
f(x_t)\}$.
    \item \vspace{-0.25cm} Update $\mu_{t,0}(x)$ and $\sigma_{t,0}(x)$ using $X_t$ and $F_t.$
    \end{enumerate}
\STATE \vspace{-0.25cm} {\bf Output}: optimization iterates $\{x_1, x_2, \ldots, x_T\}.$ 
\end{algorithmic}
\end{algorithm}

\subsection{Performance Metric}
\label{subsec:Eval_Metric}
The performance of Bayesian optimization algorithms is often analyzed through bounds on their \emph{simple regret}, $S_T$, given by
\begin{equation*}
S_T = f^* - \max_{t=1,\ldots, T} f(x_t),
\end{equation*}
or their \emph{cumulative regret}, $R_T,$ given by 
\begin{equation*}
    R_T = \sum_{t=1}^T r_t, \qquad r_t = f^* - f(x_t),
\end{equation*}
where $f^*$ is the maximum of the objective $f,$ $x_t$ is the $t$-th iterate of the optimization algorithm, and $r_t$ is called the \emph{instantaneous regret.} Notice that from the definition, we observe $S_T \le \frac{R_T}{T}$ and $S_T \le r_T$. Hence, an upper bound on $\frac{R_T}{T}$ or $r_T$ serves as an upper bound for $S_T.$ The convergence of an algorithm (in the sense of recovering the maximum of the objective) is implied by showing that $S_T \to 0$ as $T\to \infty$.

Several convergence results have been established directly through bounding the simple regret or the instantaneous regret \citep{bull2011convergence, de2012exponential}. 
On the other hand, the goal of theoretical analyses of optimization algorithms based on the cumulative regret $R_T$ is to show sublinear growth rate of $R_T$ to ensure convergence to the global maximum. 
A bound on the rate of convergence (with respect to simple regret) is then given by the decaying rate of $R_T/T$. In this context, $R_T$ serves as a useful intermediate quantity to establish convergence rates for popular Bayesian optimization strategies including algorithms based on upper confidence bounds \cite{srinivas2009gaussian, chowdhury2017kernelized, bogunovic2021misspecified} and Thompson sampling \cite{chowdhury2017kernelized}. 
As noted in \citep{bull2011convergence}, a caveat of using $R_T$ to analyze Bayesian optimization algorithms is that the fastest rate of convergence one can obtain is $T^{-1}.$ In addition, $R_T$ accounts for costs that are not incurred by the optimization algorithm. 
For these reasons, we analyze our new algorithms using simple and instantaneous regret, and additionally compare our simple regret bounds with those implied by existing bounds on $R_T$.

\subsection{Choice of Kernel}
We will consider the well-specified setting where Gaussian process interpolation for surrogate modeling is implemented using the same kernel $k$ which specifies the deterministic assumption on $f,$ namely $f \in \mathcal{H}_k$. The impact of kernel misspecification on Bayesian optimization algorithms is studied in \cite{bogunovic2021misspecified,kim2022optimization}. 

For concreteness, we focus on  \emph{Mat\'ern kernels} with smoothness parameter $\nu$ and lengthscale parameter $\ell$, given by
$$
k(x, x') = \frac{1}{\Gamma(\nu)2^{\nu-1}}\left(\frac{\sqrt{2\nu}\|x-x'\|}{\ell}\right)^\nu B_\nu \left(\frac{\sqrt{2\nu}\|x-x'\|}{\ell}\right),
$$
where $B_\nu$ is a modified Bessel function of the second kind, and on \emph{squared exponential kernels} with lengthscale parameter $\ell$, given by
$$
k(x, x') = \exp \left(-\frac{\|x-x'\|^2}{2 \ell^2}\right).
$$
We recall that the Matérn kernel converges to the squared exponential kernel in the large $\nu$ asymptotic. Both types of kernel are widely used in practice, and we refer to \cite{williams2006gaussian,wendland2004scattered, stein2012interpolation} for further background. 

\section{Related Work}\label{sec:relatedwork}
\subsection{Existing Regret Bounds: Noisy Observations}
\label{subsec:Existing_Bounds}
Numerous works have established cumulative regret bounds for Bayesian optimization with noisy observations under both deterministic assumption on the objective function \cite{srinivas2009gaussian,chowdhury2017kernelized, vakili2021information, bogunovic2021misspecified,  russo2014learning, kandasamy2018parallelised}. These bounds involve a
quantity known as the \emph{maximum information gain}, which under a Gaussian noise assumption is given by $\gamma_t = \frac{1}{2}\log |I + \lambda^{-1}K_tt|, $ where $\lambda>0$ represents the noise level.  In particular, under a deterministic objective function assumption, \cite{chowdhury2017kernelized} showed a cumulative regret bound for GP-UCB of the form $\mathcal{O}\left(\gamma_T \sqrt{T } \right)$, which improves the one obtained in \cite{srinivas2009gaussian} by a factor of $\mathcal{O}\bigl(\log^{3/2}(T)\bigr)$. By tightening existing upper bounds on the maximum information gain, \cite{vakili2021information} established a cumulative regret bound for GP-UCB of the form 
$$
R_T = \begin{cases}
    \mathcal{O}\left(T^{\frac{2\nu+3d}{4\nu+2d}}\log^{\frac{2\nu}{2\nu + d}}T\right),\\
    \mathcal{O}\left(T^{\frac{1}{2}}\log^{d+1} T\right), 
\end{cases}
$$
for Matérn and squared exponential kernels.

\subsection{Existing Regret Bounds: Noise-Free Observations}
In contrast to the noisy setting, few works have obtained regret bounds with noise-free observations. With an expected improvement acquisition function and Matérn kernel, \cite{bull2011convergence} provided a simple regret bound of the form $\tilde{\mathcal{O}}\left(T^{-\min\{\nu,1\}/d}\right)$, where $\tilde{\mathcal{O}}$ suppresses logarithmic factors, under deterministic objective function assumption. On the other hand, \cite{de2012exponential} introduced a branch and bound algorithm that achieves an exponential rate of convergence for the instantaneous regret, under the probabilistic assumption on the objective function. However, unlike the standard GP-UCB algorithm,
the algorithm in \cite{de2012exponential} requires many observations in each iteration to reduce the search space, and it further requires solving a constrained optimization problem in the reduced search space.

To the best of our knowledge, 
\cite{lyu2019efficient} presents the only cumulative regret bound available for GP-UCB with noise-free observations under a deterministic assumption on the objective. Specifically, they consider Algorithm \ref{noiseless_BO}, and, noticing that $\sigma_{t,0}(x) \le \sigma_{t,\lambda}(x)$ for any $\lambda \ge 0$, they deduce that existing cumulative regret bounds for Bayesian optimization with noisy observations remain valid with noise-free observations. Furthermore, in the noise-free setting, the cumulative regret bound is improved by a factor of $\sqrt{\gamma_T}$, which comes from using a constant weight parameter $\beta_t := \|f\|_{\mathcal{H}_k(\mathcal{X})}^2$ given by the squared RKHS norm of the objective. This leads to a cumulative regret bound with rate $\mathcal{O}(\sqrt{\gamma_T T})$, which gives  
\begin{align}\label{eq:boundLyu}
R_T = \begin{cases}
    \mathcal{O}\left(T^{\frac{\nu+d}{2\nu+d}}\log^{\frac{\nu}{2\nu + d}}T\right),\\
    \mathcal{O}\left(T^{\frac{1}{2}}\log^{\frac{d+1}{2}} T\right), 
\end{cases}
\end{align}
for Matérn and squared exponential kernels. 

\subsection{Tighter Cumulative Regret Bound for Squared Exponential Kernels} \cite{vakili2022open} sets as an open problem whether one can improve
the cumulative regret bounds in \eqref{eq:boundLyu} for the GP-UCB algorithm with noise-free observations. For squared exponential kernels, we claim that one can further improve the cumulative regret bound in \eqref{eq:boundLyu} by a factor of $\sqrt{\log T}$.
\begin{theorem}\label{UCB_cum_reg_bd}
Let $f \in \mathcal{H}_k(\mathcal{X})$, where $k$ is a squared exponential kernel. GP-UCB with noise-free observations and $\beta_t := \|f\|_{\mathcal{H}_k}^2$  satisfies the cumulative regret bound 
$$
R_T = \mathcal{O}\left(T^{\frac{1}{2}}\log^{\frac{d}{2}}  T \right),
$$
which yields the convergence rate of
$$
S_T = \mathcal{O}\left(T^{-\frac{1}{2}}\log^{\frac{d}{2}} T \right).
$$
\end{theorem}
\begin{remark}
Our improvement in the bound comes from a constant term $\frac{1}{\log(1+\lambda^{-1})}$, which was ignored in existing analyses with noisy observations. By letting $\lambda \to 0$, the constant offsets a $\sqrt{\log T}$ growth in the cumulative regret bound.
\end{remark}

\begin{remark}
For Matérn kernels, a similar approach to improve the rate is not feasible. A state-of-the-art, near-optimal upper bound on the maximum information gain with Matérn kernels obtained in \cite{vakili2021information} introduces a polynomial growth factor as the noise variance $\lambda$ decreases to zero. Minimizing the rate of an upper bound in \cite{vakili2021information} one can match the rate obtained in \cite{lyu2019efficient}. 
\end{remark}

\begin{remark}
Unlike cumulative regret bounds with noisy observations, Theorem \ref{UCB_cum_reg_bd} and the results in \cite{lyu2019efficient} are deterministic. 
\end{remark}

\subsection{Optimal Simple Regret Bounds}
Under the deterministic objective function assumption, Theorem \ref{UCB_cum_reg_bd} refines the rate bound in \eqref{eq:boundLyu} for GP-UCB with noise-free observations using squared exponential kernels. In the rest of the paper, we will design new algorithms that achieve drastically faster convergence rates. In particular, for Matérn kernels, our algorithms nearly achieve the optimal convergence rate in \citep{bull2011convergence} of the form 
$
S_T = \Theta \left(T^{-\frac{\nu}{d}} \right).
$
Furthermore, our algorithms' convergence rate is competitive to algorithms that satisfy the conjectured cumulative regret bound in \cite{vakili2022open} of the form
$$
R_T = \begin{cases}
    \mathcal{O}(T^{\frac{d-\nu}{d}}), & \text{for} ~ d > \nu, \\
    \mathcal{O}(\log T), & \text{for} ~ d = \nu,  \\
    \mathcal{O}(1), & \text{for} ~ d < \nu, 
\end{cases}
$$
which translates to the convergence rate of
$$
S_T = \begin{cases}
    \mathcal{O}\left(T^{-\frac{\nu}{d}}\right), & \text{for} ~ d > \nu, \\
    \mathcal{O}\left(T^{-1}\log T\right), & \text{for} ~ d = \nu, \\ 
    \mathcal{O}\left(T^{-1}\right), & \text{for} ~ d < \nu.
\end{cases}
$$
Our new algorithms nearly achieve the optimal convergence rate while preserving the ease of implementation of GP-UCB algorithms. The recent preprint \cite{salgia2023random} proposes an alternative batch-based approach, which combines random sampling with domain shrinking to attain the conjectured cumulative regret bounds for $d \ge \nu$ with a high probability.

\section{Exploitation with Accelerated Exploration} 
\label{Sec:Decouple}%
 \subsection{How Well Does GP-UCB Explore?}\label{ssec:howwell}
 The GP-UCB algorithm selects query points by optimizing an acquisition function which incorporates the posterior mean to promote exploitation and the posterior standard deviation to promote exploration. Our new algorithms are inspired by the desire to improve the exploration of GP-UCB. Before introducing the algorithms in the next subsection, we heuristically explain why such an improvement may be possible.

A natural way to quantify how well data $X_t = \{x_1, \ldots, x_t\}$ cover the search space $\mathcal{X}$ is via the \emph{fill-distance}, given by
$$
h(\mathcal{X}, X_t) \coloneqq \sup_{x \in \mathcal{X}} \inf_{x_i \in X_t} \|x-x_i\|.
$$
The fill-distance appears in error bounds for Gaussian process interpolation and regression \cite{wendland2004scattered,teckentrup2020convergence, stuart2018posterior, tuo2020kriging}. For quasi-uniform points, it holds that $
h(\mathcal{X}, X_t) = \Theta\left(t^{-\frac{1}{d}}\right)$, which is the fastest possible decay rate for any sequence of design points. The fill-distance of the query points selected by our new Bayesian optimization algorithms will (nearly) decay at this rate. 

\begin{figure}[htp]
    \centering
\includegraphics[height=.3\textwidth]{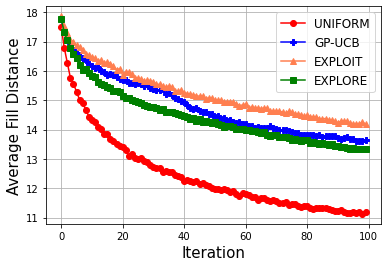}
    \caption{Average fill-distance of a set of query points obtained using four different algorithms over 100 independent experiments. The results are based on a 10-dimensional Rastrigin function. The discrete subset $\mathcal{X}_D$ consists of 100 Latin hypercube samples.}
    \label{EXAMPLE_FILL_RAST}
\end{figure}

\cite{wenzel2021novel} introduced a stabilized greedy algorithm to obtain query points by maximizing the posterior predictive standard deviation at each iteration. Their algorithm sequentially generates a set of query points whose fill-distance decays at a rate $\Theta\left(t^{-\frac{1}{d}}\right)$ by sequentially solving constrained optimization problems, which can be computationally demanding. Since GP-UCB simply promotes exploration through the posterior predictive standard deviation term in the UCB acquisition function, one may heuristically expect the fill-distance of query points selected by the standard GP-UCB algorithm to decay at a slower rate. On the other hand, 
a straightforward online approach to obtain a set of query points whose fill-distance nearly decays at a rate $\Theta\left(t^{-\frac{1}{d}}\right)$ is to sample randomly from a probability measure $P$ with a strictly positive Lebesgue density on $\mathcal{X}.$ Specifically, \cite{oates2019convergence} shows that, in expectation, the fill-distance of independent samples from such a measure decays at a near-optimal rate: for any $\epsilon>0,$ $\mathbb{E}_P \bigl[h(\mathcal{X}, X_t)\bigr] = \mathcal{O}(t^{-\frac{1}{d}+ \epsilon}).$

Figure \ref{EXAMPLE_FILL_RAST} compares the decay of the fill-distance for query points selected using four strategies. For a 10-dimensional Rastrigin function, we consider: (i) GP-UCB; (ii)  EXPLOIT  (i.e., maximizing the posterior mean at each iteration); (iii)  EXPLORE (i.e., maximizing the posterior variance); and (iv) UNIFORM (i.e. independent uniform random samples on $\mathcal{X}$). The results were averaged over 100 independent experiments. The fill-distance for GP-UCB lies in between those for EXPLORE and EXPLOIT; whether it lies closer to one or the other depends on the choice of weight parameter, which here we choose based on a numerical approximation of the max-norm of the objective, $\beta_t^{\frac12} = \max_{x \in \mathcal{X}_D}|f(x)|$ where $\mathcal{X}_D$ is a discretization of the search space $\mathcal{X}$. Note that UNIFORM yields a drastically smaller fill-distance even when compared with EXPLORE. Our new algorithms will leverage random sampling to enhance exploration in Bayesian optimization and achieve improved regret bounds.

\subsection{Improved Exploration via Random Sampling}
In this subsection, we introduce two Bayesian optimization algorithms that leverage random sampling as a tool to facilitate efficient exploration of the search space and enhance the accuracy of surrogate models of the objective function with which to acquire new optimization iterates. While the GP-UCB algorithm selects a single query point $x_t$ per iteration, our algorithms select two query points $\{x_t,  \tilde{x}_t\}$ to produce a single optimization iterate $x_t$.

The first algorithm we introduce, which we call GP-UCB+, selects a query point $x_t$ using the GP-UCB acquisition function and another query point $ \tilde{x}_t$ by random sampling. We outline the pseudocode in Algorithm \ref{GP_UCB_Exp_Random}. The second algorithm we introduce, which we call EXPLOIT+, decouples the exploitation and exploration goals, selecting one query point $x_t$ by maximizing the posterior mean to promote exploitation, and another query point $ \tilde{x}_t$ by random sampling to promote exploration. We outline the pseudocode in Algorithm \ref{Pure_Exp_Random}. As noted above, both GP-UCB+ and EXPLOIT+ produce a single optimization iterate $x_t;$ the role of the additional query point $\tilde{x}_t$ is to enhance the surrogate model of the objective with which $x_t$ is acquired. 
Since our new algorithms require two query points per iteration, in our numerical experiments in Section \ref{sec:numerics} we ensure a fair comparison by running them for has as many iterations as used for algorithms that require one query point per iteration. For the convergence rate analysis in Subsection \ref{subse:reg_bound_det}, the fact that our algorithms require twice as many iterations is inconsequential, since halving the number of iterations does not affect the convergence rate.

\begin{algorithm}[htp]
\caption{GP-UCB+\label{GP_UCB_Exp_Random}.} 
\begin{algorithmic}[1]
    \STATE {\bf Input}: Kernel $k;$ Total number of iterations $T$; Initial design points $X^{\text{full}}_0$; Initial noise-free observations $F^{\text{full}}_0$; Probability distribution $P$ on $\mathcal{X}$; Weights $\{\beta_t\}_{t=1}^{T}.$
    \STATE Construct posterior mean $\mu^{\text{full}}_{0,0}(x)$ and standard deviation $\sigma^{\text{full}}_{0,0}(x)$ using $X^{\text{full}}_0$ and $F^{\text{full}}_0$.
    \STATE {\bf For} $t = 1, \ldots, T$ {\bf do}: 
    \begin{enumerate}
     \setlength{\belowdisplayskip}{0pt} \setlength{\belowdisplayshortskip}{0pt}
\setlength{\abovedisplayskip}{10pt} \setlength{\abovedisplayshortskip}{0pt}
    \item \vspace{-0.25cm} {\bf Exploitation + Exploration:} 
    Set $$x_t = \arg\max_{x \in \mathcal{X}} \mu^{\text{full}}_{t-1,0}(x) + \beta_t^{\frac12} \sigma^{\text{full}}_{t-1,0}(x).$$ 
    \item {\bf Exploration:} Sample $\tilde x_{t} \sim P.$ 
    \item \vspace{-0.25cm}  Set 
         \setlength{\belowdisplayskip}{0pt} \setlength{\belowdisplayshortskip}{5pt}
\setlength{\abovedisplayskip}{0pt} \setlength{\abovedisplayshortskip}{0pt}
    \begin{align*}
        X^{\text{full}}_t &= X^{\text{full}}_{t-1} \cup \{
x_{t}, \tilde x_{t}\}, \\
F^{\text{full}}_t &= F^{\text{full}}_{t-1} \cup \{
f(x_{t}), f(\tilde x_{t})\}.
    \end{align*}
    \item Update $\mu^{\text{full}}_{t,0}(x)$ and $\sigma^{\text{full}}_{t,0}(x)$ using $X^{\text{full}}_t$ and $F^{\text{full}}_t.$
    \end{enumerate}
\STATE \vspace{-0.25cm}  {\bf Output}: optimization iterates $\{x_1, x_2, \ldots, x_T\}.$ 
\end{algorithmic}
\end{algorithm}

\begin{algorithm}[htp]
\caption{EXPLOIT+\label{Pure_Exp_Random}.} 
\begin{algorithmic}[1]
    \STATE {\bf Input}: Kernel $k;$ Total number of iterations $T$; Initial design points $X^{\text{full}}_0$; Initial noise-free observations $F^{\text{full}}_0$; Probability distribution $P$ on $\mathcal{X}$. 
    \STATE Construct posterior mean $\mu^{\text{full}}_{0,0}(x)$ and standard deviation $\sigma^{\text{full}}_{0,0}(x)$ using $X^{\text{full}}_0$ and $F^{\text{full}}_0$.
    \STATE {\bf For} $t = 1, \ldots, T$ {\bf do}: 
    \begin{enumerate}
    \item \vspace{-0.25cm} {\bf Exploitation:} Set $x_t = \arg\max_{x \in \mathcal{X}} \mu^{\text{full}}_{t-1,0}(x).$
    \item \vspace{-0.25cm} {\bf Exploration:} Sample $\tilde x_{t} \sim P.$
    \item \vspace{-0.25cm} Set 
        \setlength{\belowdisplayskip}{10pt} \setlength{\belowdisplayshortskip}{5pt}
\setlength{\abovedisplayskip}{0pt} \setlength{\abovedisplayshortskip}{0pt}
    \begin{align*}
        X^{\text{full}}_t &= X^{\text{full}}_{t-1} \cup \{
x_{t}, \tilde x_{t}\}, \\
F^{\text{full}}_t &= F^{\text{full}}_{t-1} \cup \{
f(x_{t}), f(\tilde x_{t})\}.
    \end{align*}
    \item \vspace{-0.25cm} Update $\mu^{\text{full}}_{t,0}(x)$ using $X^{\text{full}}_t$ and $F^{\text{full}}_t.$
    \end{enumerate}
\STATE \vspace{-0.25cm} {\bf Output}: optimization iterates $\{x_1, x_2, \ldots, x_T\}.$ 
\end{algorithmic}
\end{algorithm}

Notably, EXPLOIT+ does not require input weight parameters $\{\beta_t\}_{t=1}^T.$ As mentioned in Section \ref{sec:relatedwork}, many regret bounds for GP-UCB algorithms rely on choosing the weight parameters as the squared RKHS norm of the objective or in terms of a bound on it. The performance of GP-UCB and GP-UCB+ can be sensitive to this choice, which in practice is often based on empirical tuning or heuristic arguments rather than guided by the theory. In contrast, EXPLOIT+ achieves the same regret bounds as GP-UCB+ and drastically faster rates than GP-UCB without requiring the practitioner to specify weight parameters. Additionally, EXPLORE+ shows competitive empirical performance.

\begin{remark}
In the exploration step, one can acquire a batch of points to further enhance the exploration of GP-UCB+ and EXPLOIT+. As long as the number of points sampled at each iteration does not grow with respect to the iteration index $t$, the regret bounds stated in Theorem \ref{thm:CumulativeRegretBounds} below remain valid.
\end{remark}

\begin{remark}
A common heuristic strategy to expedite the performance of Bayesian optimization algorithms is to acquire a moderate number of initial design points by uniformly sampling the search space. Since the order of the exploration and exploitation steps can be swapped in our algorithms, such heuristic strategy can be interpreted as an initial batch exploration step. 
\end{remark}

\begin{remark}
A natural choice for $P$ is the uniform distribution on the search space $\mathcal{X}$. Our theory, which utilizes bounds on the fill-distance of randomly sampled query points from \cite{oates2019convergence}, holds as long as $P$ has a strictly positive  Lebesgue density on $\mathcal{X}.$ In what follows, we assume throughout that $P$ satisfies this condition.
\end{remark}

\subsection{Regret Bounds}\label{subse:reg_bound_det}
We now obtain regret bounds under the deterministic assumption that $f$ belongs to the RKHS of a kernel $k.$ 
Our algorithms are random due to sampling from $P,$ and we show cumulative regret bounds in expectation with respect to such randomness.
\begin{theorem}\label{thm:CumulativeRegretBounds}
Let $f \in \mathcal{H}_k(\mathcal{X}).$ Suppose $t \in \mathbb{N}$ is large enough. GP-UCB+ with $\beta_t:= \|f\|_{\mathcal{H}_k(\mathcal{X})}^2$ and EXPLOIT+ attain the following instantaneous regret bounds. For Matérn kernels with parameter $\nu > 0$,
$$
\mathbb{E}_P[r_t] = 
    \mathcal{O}\left(t^{-\frac{\nu}{d}+\varepsilon}\right)
$$
where $\varepsilon > 0$ can be arbitrarily small. For squared exponential kernels,
$$
\mathbb{E}_P[r_t] = \mathcal{O}\left(\exp\left(-C t^{\frac{1}{d}-\varepsilon}\right)  \right),
$$
for some constant $C > 0$ with an arbitrarily small $\varepsilon > 0$.
\end{theorem}
\begin{remark}
In particular, Theorem \ref{thm:CumulativeRegretBounds} implies that
$$
\mathbb{E}_P[S_T] = \begin{cases}
      \mathcal{O}\left(T^{-\frac{\nu}{d}+\varepsilon}\right),  & \text{for Matérn kernels}, \\
      \mathcal{O}\left(\exp\left(-C T^{\frac{1}{d}-\varepsilon}\right)  \right),  & \text{for squared exponential kernels}, 
\end{cases}
$$
for some constant $C > 0.$
\end{remark}
\begin{remark}
For Matérn kernels, the proposed algorithms nearly attain in expectation the optimal convergence rate established by \citep{bull2011convergence} and the convergence rate implied by the cumulative regret bound conjectured in \cite{vakili2022open}. Moreover,  one can further obtain the exact optimal convergence rate by replacing the random sampling step in GP-UCB+ and EXPLOIT+ with a more computationally expensive quasi-uniform sampling scheme.
\end{remark}
\begin{remark}
    Compared with the GP-UCB algorithm with noise-free observations, the proposed algorithms attain improved convergence rates in expectation for both Matérn and squared exponential kernels. For the Matérn kernel, the new convergence rate has a faster polynomial decaying factor with a removal of the logarithmic growth factor. For the squared exponential kernel, the proposed algorithms have an exponential convergence rate, whereas the improved bound for the GP-UCB algorithm in Theorem \ref{UCB_cum_reg_bd} has a convergence rate of $\mathcal{O}(T^{-\frac{1}{2}}\log^{\frac{d}{2}}(T)$).
\end{remark}

\begin{remark}
    For Matérn kernels, compared with the recent preprint \cite{salgia2023random}, which attains $\mathcal{O}(T^{-1}\log^{\frac{3}{2}}T)$ convergence rate when $d < \nu$, our algorithms attains exponential convergence rate. When $d \ge \nu$, \cite{salgia2023random} attains the convergence rate implied by the conjecture in \cite{vakili2022open} up to a logarithmic factor, while we attain the implied convergence rate up to a factor of $\mathcal{O}(T^\varepsilon)$, for arbitrarily small $\varepsilon > 0$. Compared to existing works, we additionally establish exponential convergence rates with the squared exponential kernel.
\end{remark}

\begin{remark}
    Our theoretical analysis largely ignores the constant factor which depends on the dimension $d$ of the search space. Based on our experience with a wide range of numerical experiments, to see the clear effect of a faster convergence rate, we recommend using the proposed strategies in five or larger-dimensional problems. In low-dimensional problems, distinctly faster convergence was not observed in comparison to other Bayesian optimization strategies.
\end{remark}

\section{Enhanced Surrogates for Bayesian Inference}\label{sec:postapproxandsampling}

In this section, we treat optimization iterates as design points to build a Gaussian process surrogate model for an unnormalized log-posterior. When the points chosen by GP-UCB are used to build the Gaussian process surrogate, these points are typically highly concentrated around the mode of the posterior distribution. Therefore, after normalization, the resulting posterior approximation often reflects poorly the global shape of the true posterior. On the other hand, if the Gaussian process surrogate is built using randomly sampled design points, the surrogate may not serve as an accurate proxy of the true posterior in the region around the mode. 

Resolving the weakness of each approach, surrogates based on EXPLOIT+ and GP-UCB+ portray a more accurate depiction of the posterior in the region around its mode, while still accurately capturing its global shape. The proposed Bayesian optimization strategies, therefore, serve as experimental design tools to build a Gaussian process surrogate that successfully captures both the local and global shape of the log-posterior. Our claim is supported by a convergence rate analysis as well as by several numerical examples. In particular, we provide a computationally efficient approximate Bayesian inference strategy based on EXPLOIT+ and GP-UCB+ for intractable likelihood functions arising in parameter inference for differential equations.

Let us denote the target posterior density as $\pi(x) \propto \exp \bigl(V(x)\bigr)$, where $V(x)$ is an unnormalized log-posterior, also known as the energy function. The normalizing constant is given by $Z_\pi = \int_{\mathcal{X}}  \exp \bigl(V(x)\bigr) \, dx$. We denote by $\mu_t^V$ the Gaussian process posterior mean function for the energy function $V$ based on design points $X_t^{\text{full}}$. We denote the posterior density based on this Gaussian process surrogate by
$$
\pi_t(x) = \frac{1}{Z_{\pi,t}}\exp \Bigl(\mu_t^V(x)\Bigr),
$$
where $Z_{\pi,t} = \int_{\mathcal{X}} \exp \bigl(\mu_t^V(x)\bigr)dx$. With these notations, we first show that $Z_{\pi,t}$ is bounded. Next, as the proposed algorithms EXPLOIT+ and GP-UCB+ involve random sampling, we establish an upper bound on the expected Hellinger distance between the two quantities. In addition, utilizing the established approximation bound, the expected Hellinger distance between the true posterior density and the surrogate posterior density will be obtained in Theorem \ref{thm:pos_approx_log_unnorm} below.

\begin{proposition}\label{prop:norm_const_bound}
Let $V \in \mathcal{H}_k (\mathcal{X})$ with $k$ being Mat\'ern or squared exponential kernel. Suppose $\mu_t^V(x)$ is a Gaussian process surrogate of $V$ based on Bayesian optimization strategies with random exploration. Then, there exist positive constants $C_1$ and $C_2$ such that for all $t \in \mathbb{N}$, 
$$
C_1 \le Z_{\pi, t} \le C_2.
$$
\end{proposition}

\begin{remark}
Unlike the results in \cite{stuart2018posterior}, which build Gaussian process surrogates for a forward map or a likelihood term, our result is based on Gaussian process surrogates for an unnormalized log-posterior density that incorporates both the likelihood function and the prior distribution.
\end{remark}

\begin{theorem}\label{thm:pos_approx_log_unnorm}
Let $V \in \mathcal{H}_k (\mathcal{X})$ with $k$ being Mat\'ern or squared exponential kernel. Suppose $\mu_t^V(x)$ is a Gaussian process surrogate of $V$ based on Bayesian optimization strategies with random exploration. Suppose the target posterior density is given by $\pi(x) \propto \exp \bigl(V(x)\bigr).$ Then, for sufficiently large $t \in \mathbb{N}$, with the Mat\'ern kernel with a smoothness parameter $\nu > 0$,
$$
\mathbb{E}_P[d_{H}(\pi, \pi_t)] =  
    \mathcal{O}\left(t^{- \frac{\nu}{d} - \frac{1}{2} + \epsilon} \right), 
$$
and with the squared exponential kernel, 
$$
\mathbb{E}_P \left[d_H(\pi, \pi_t) \right] =  \mathcal{O}\left(\exp \left(-C t^{\frac{1}{d}-\epsilon}\right) \right),
$$
for some constant $C > 0$ with an arbitrarily small $\epsilon > 0$.
\end{theorem}

\begin{remark}
In comparison to the posterior approximation results in \cite{stuart2018posterior}, whose results are stated in terms of the fill distance, Theorem \ref{thm:pos_approx_log_unnorm} has an explicit dependence on the number of design points. For Mat\'ern kernels, the established results are near-optimal in the sense that they almost match in expectation the approximation rates obtained in \cite{stuart2018posterior} after plugging in the optimal fill-distance decaying rate. Our results also cover the squared exponential kernel, which was not considered in \cite{stuart2018posterior}.
\end{remark}

\begin{remark}
For Mat\'ern kernels, if our posterior approximation is used in the context of normalizing constant approximation, our algorithms nearly obtain the best possible worst-case approximation rate proven in \cite{novak2006deterministic}. For squared exponential kernels, our algorithms nearly match the approximation rate obtained by the adaptive Bayesian quadrature rule proposed in \cite{kanagawa2019convergence}. If the random sampling step in Algorithms \ref{GP_UCB_Exp_Random} and  \ref{Pure_Exp_Random} is replaced by a quasi-uniform sampling scheme, the approximation rate is analogous to the rates established in \cite{novak2006deterministic, kanagawa2019convergence} with the potential of sacrificing computational efficiency. 
\end{remark}

\section{Numerical Experiments}\label{sec:numerics}
This section explores the empirical performance of our methods. We consider 1) three benchmark optimization tasks, 2) hyperparameter tuning for a machine learning model, 3) optimizing a black-box objective function designed to guide engineering decisions, and 4) posterior approximation for Bayesian inference of differential equations. For optimization tasks, we compare the new algorithms (GP-UCB+, EXPLOIT+) with GP-UCB and two other popular Bayesian optimization strategies: Expected Improvement (EI) and Probability of Improvement (PI). We also compare with the EXPLOIT approach outlined in Subsection \ref{ssec:howwell}, but not with EXPLORE as this method did not achieve competitive performance. For the posterior approximation task, we compare the new algorithms with baseline algorithms that use design points obtained by GP-UCB and by random sampling.  Throughout, we choose the distribution $P$ that governs random exploration in the new algorithms to be uniform on $\mathcal{X}$. For the weight parameter of UCB acquisition functions, we considered a well-tuned constant value $\beta_t^{1/2} = 2$ that achieves good performance in our examples, and the approach in \cite{ray2019bayesian}, which sets  $\beta_t^{1/2} = \max_{x \in \mathcal{X}_D} |f(x)|$ where $\mathcal{X}_D$ is a discretization of the search space. All the hyperparameters of the kernel function were iteratively updated through maximum likelihood estimation. Since the new algorithms need two noise-free observations per iteration but the methods we compare with only need one, we run the new algorithms for half as many iterations to ensure a fair comparison.

\subsection{Benchmark Objective Functions}\label{subsec:BENCH}
We consider three 10-dimensional benchmark objective functions: Ackley, Rastrigin, and Levy. 
Each of them has a unique global maximizer but many local optima, posing a challenge to 
 standard first and second-order convex optimization algorithms. Following the virtual library of simulation experiments: \url{https://www.sfu.ca/~ssurjano}, we respectively set the search space to be $[-32.768, 32.768]^{10}, [-5.12, 5.12]^{10},$ and $[-10, 10]^{10}$. We used a Matérn kernel with the default initial smoothness parameter $\nu = 2.5$ and initial lengthscale parameter $\ell = 1$. For each method and objective, we obtain 400 noise-free observations and average the results over 20 independent experiments. For GP-UCB and GP-UCB+, we set $\beta_t^{1/2} =2.$ Figure \ref{Benchmark_SIMPLE_PLOT} shows the average simple regrets, given by $f^* - \max_{t=1,\ldots, T} f(x_t).$ We report the regret as a function of the number of observations rather than the number of iterations to ensure a fair comparison. For all three benchmark functions, GP-UCB+ and EXPLOIT+ outperform the other methods. To further demonstrate the strength of the proposed algorithms, Table \ref{table:benchmarks} shows the average simple  regret at the last iteration, normalized so that for each benchmark objective the worst-performing algorithm has unit  simple regret. 
 Table \ref{table:benchmarksstd} in Appendix \ref{appendix:BENCHMARK} shows results for the standard deviation, indicating that the new methods are not only more accurate, but also more precise.

To illustrate the sensitivity of UCB algorithms to the choice of weight parameters, we include numerical results with $\beta_t^{1/2} = \max_{x \in \mathcal{X}_D} |f(x)|$ in Appendix \ref{appendix:BENCHMARK}. In particular, since GP-UCB+ has an additional exploration step through random sampling,  using a smaller weight parameter for GP-UCB+ than for GP-UCB tends to work more effectively. Remarkably, the parameter-free EXPLOIT+ algorithm achieves competitive performance compared with UCB algorithms with well-tuned weight parameters.

\begin{figure}[htp]
  \centering
\includegraphics[scale=0.35]{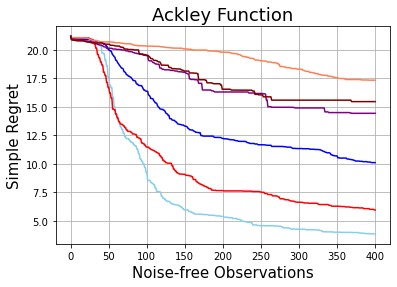}
\includegraphics[scale=0.35]{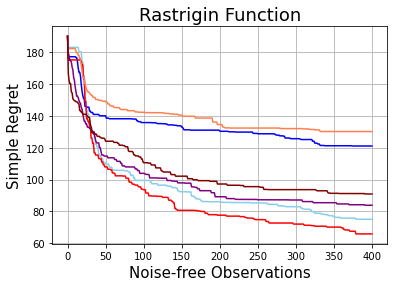}
\includegraphics[scale=0.35]{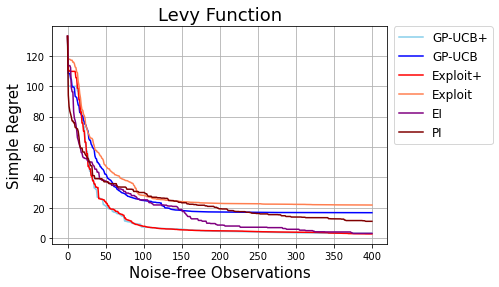}
  \caption{Simple regret vs number of noise-free observations.}
  \label{Benchmark_SIMPLE_PLOT}
\end{figure}

\begin{table}[htp]
\begin{center}
\begin{sc}
\begin{tabular}{lcccr}
\toprule
Method & Ackley & Rastrigin & Levy  \\
\midrule
GP-UCB+ & \textbf{0.222} & 0.576 & 0.146  \\
GP-UCB & 0.583 & 0.930 & 0.768  \\
EXPLOIT+  & 0.342 & \textbf{0.505} & \textbf{0.126}  \\
EXPLOIT  & 1.000 & 1.000 & 1.000   \\
EI   & 0.832 & 0.644 & 0.142 \\
PI & 0.891 & 0.698 & 0.507  \\
\bottomrule
\end{tabular}
\end{sc}
\end{center}
\vskip -0.1in
\caption{\label{table:benchmarks} Normalized average simple regret with 400 function evaluations for benchmark objectives in dimension $d = 10.$}
\end{table}

\subsection{Random Forest Hyperparameter Tuning}\label{subsec:RF_HYPER}
Here we use Bayesian optimization to tune four hyperparameters of a random forest regression model for the California housing dataset \cite{pace1997sparse}. The parameters of interest are (i) three integer-valued quantities: the number of trees in the forest, the maximum depth of the tree, and the minimum number of samples required to split the internal node; and (ii) a real-valued quantity between zero and one: the transformed maximum number of features to consider when looking for the best split.  For the discrete quantities, instead of optimizing over a discrete search space, we performed the optimization over a continuous domain and truncated the decimal values when evaluating the objective function. We split the dataset into training (80\%) and testing (20\%) sets. To define a deterministic objective function, we fixed the random state parameter for the \texttt{RandomForestRegressor} function from the Python scikit-learn package and built the model using the training set. We then defined our objective function to be the negative mean-squared test error of the built model. We used a Matérn kernel with initial smoothness parameter $\nu = 2.5$ and initial lengthscale parameter $\ell = 1$. For the GP-UCB and GP-UCB+ algorithms, we set $\beta_t^{1/2} = \max_{x \in \mathcal{X}_D} |f(x)|$ where $\mathcal{X}_D$ consists of 40 Latin hypercube samples. We conducted 20 independent experiments with 80 noise-free observations. From Table \ref{table:HYPER} and Figure \ref{EXAMPLE_CUM_RF}, we see that both GP-UCB+ and EXPLOIT+ algorithms led to smaller cumulative test errors. An instantaneous  test error  plot with implementation details can be found 
in Appendix \ref{appendix:RF}.

\begin{figure}
\begin{floatrow}
\ffigbox{%
  \includegraphics[height=.25\textwidth]{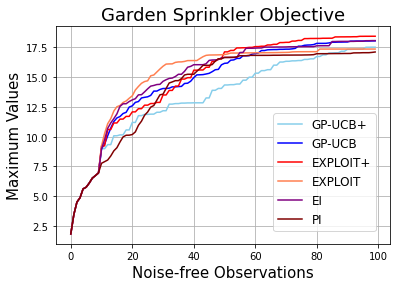}%
}{%
  \caption{\label{EXAMPLE_CUM_RF}Cumulative test error vs number of noise-free observations.}%
}
\capbtabbox{%

  \begin{sc}
\begin{tabular}{lcr}
\toprule
Method & Mean $\pm$ SD  \\
\midrule
GP-UCB+ &  $28.552 \pm 1.971$ & \\
GP-UCB  &  $35.226 \pm 1.467$ & \\
\textbf{EXPLOIT+} & $\mathbf{26.346 \pm 1.404}$&  \\
EXPLOIT  & $35.026 \pm 1.092$ &    \\
EI   & $34.294 \pm 0.987$  &       \\
PI   & $33.438 \pm 1.085$ &        \\
\bottomrule
\end{tabular}
\end{sc}
}{%
  \caption{\label{table:HYPER} Cumulative test error averaged over 20 experiments.}%
}
\end{floatrow}
\end{figure}

\subsection{Garden Sprinkler Computer Model}\label{subsec:Sprinkler}
The Garden Sprinkler computer model simulates the range of a garden sprinkler that sprays water. The model contains eight physical parameters that represent vertical nozzle angle, tangential nozzle angle, nozzle profile, diameter of the sprinkler head, dynamic friction moment, static friction moment, entrance pressure, and diameter flow line. First introduced in \cite{siebertz2010statistische} and later formulated into a deterministic black-box model by \cite{pourmohamad2020compmodels}, the goal is to maximize the accessible range of a garden sprinkler over the domain of the eight-dimensional parameter space. In this problem, the observations of the objective are noise-free. Following \cite{pourmohamad2021bayesian}, for GP-UCB and GP-UCB+ we set $\beta_t^{1/2} = 2$  and used a squared exponential kernel with an initial lengthscale parameter $\ell^2 = 50$. We ran 30 independent experiments, each with 100 noise-free observations. The results in Table \ref{table:GARDEN_SPRINGK} and Figure \ref{EXAMPLE_MAX_SPRING}  demonstrate that the new algorithms achieve competitive performance. In particular, EXPLOIT+ attains on average the largest maximum value, while also retaining a moderate standard deviation across experiments.

\begin{figure}
\begin{floatrow}
\ffigbox{%
  \includegraphics[height=.25\textwidth]{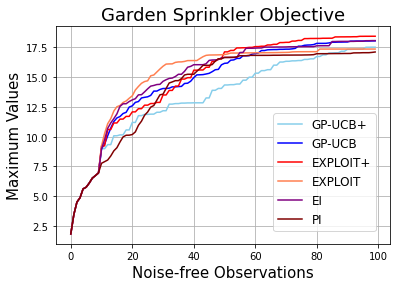}%
}{%
  \caption{\label{EXAMPLE_MAX_SPRING}Maximum attained value of the garden sprinkler objective function vs number of noise-free observations.}%
}
\capbtabbox{%

  \begin{sc}
\begin{tabular}{lcr}
\toprule
Method & Mean $\pm$ SD \\
\midrule
GP-UCB+  & 17.511 $\pm$ 1.603 & \\
GP-UCB & 18.038 $\pm$ 2.026 & \\
\textbf{EXPLOIT+}    & $\mathbf{18.427 \pm 1.825}$ &  \\
EXPLOIT    & 17.352 $\pm$ 2.537 &    \\
EI   & 18.061 $\pm$ 1.657 &       \\
PI   & 17.105 $\pm$ 2.329 &        \\
\bottomrule
\end{tabular}
\end{sc}
}{%
  \caption{\label{table:GARDEN_SPRINGK} Maximum attained value of the garden sprinkler objective function averaged over 30 experiments.}%
}
\end{floatrow}
\end{figure}

\subsection{Bayesian Inference for Parameters of Differential Equations}
In this section, our goal is no longer to optimize an objective function, but instead to construct a computationally efficient Gaussian process surrogate for an intractable unnormalized log-posterior density. To this end, we utilize a set of query locations selected by Bayesian optimization strategies as design points to build the Gaussian process surrogate. We closely follow the settings considered in \cite{schneider2017earth,cleary2021calibrate}, where the intractable target posterior distribution is over parameters of differential equations. 

Consider the following dynamical system governed by the set of differential equations given by
\begin{equation}\label{diffEQ}
\frac{dz}{dt} = F(z(t), x), \qquad  z(0) = z_0,
\end{equation}
where $x \in \mathcal{X} \subset \mathbb{R}^d$ is a parameter associated with the dynamics $z(t) \in \mathbb{R}^k$ for all time $t \ge 0$. Let us define a solution mapping $\mathcal{S}$, 
$$
\mathcal{S}: x \mapsto z(t), ~t \in [0, T],
$$
which maps each parameter value to a particular solution path of the differential equations. Here, the time interval $[0, T]$ is a prespecified time-window of interest. For a predetermined degree $m \in \mathbb{N}$, we also define an averaging operator $\mathcal{A}$, given by
$$
\mathcal{A}: z(t) \mapsto \frac{1}{T}\left(\int_{0}^Tz(t)dt, \int_{0}^Tz^2(t)dt, \ldots, \int_{0}^Tz^m(t)dt \right) \in \mathbb{R}^M,
$$
where the integral is understood as a componentwise operation. Note that $z^m(t)$ contains not only $m$-th power of its component, but also $m$-th order product between its components. In short, the averaging operator is a mapping from dynamics to a lower-dimensional summary of the dynamics: in this case, the average over time. 

The forward map of our interest is defined to be
$$
\mathcal{G}(x) = \mathcal{A} \circ \mathcal{S}(x),
$$
which maps a parameter value to a vector of moments of the dynamics. We assume the observed data is given by
$$
\mathcal{D} = \mathcal{G}(x^*) + \eta,
$$
for some true parameter $x^*\in \mathcal{X} \subset \mathbb{R}^{d}$ and $\eta\in \mathbb{R}^M$ represents a zero-centered random observational noise. In other words, the data at hand is partial information of the dynamics, which are the first $m$-th order moments in this case. Denoting the Lebesgue densities of measurement error and prior distribution by $f$ and $p$, the unnormalized log-posterior density is given by
\begin{align*}
V(x) = \log f \bigl(\mathcal{D}-\mathcal{G}(x)\bigr) + \log p(x),
\end{align*}
which we aim to replace with a computationally efficient Gaussian process surrogate. In particular, in the case of Gaussian noise and prior assumption, the unnormalized log-posterior density is given by
\begin{align*}
V(x) = -\|\mathcal{D} - \mathcal{G}(x)\|_{\Gamma}^2 - \|x-m_0\|_P^2,
\end{align*}
where $m_0$ is the prior mean, and $\Gamma$ and $P$ are the noise and prior covariances. In simulations, following \cite{schneider2017earth}, we set $\Gamma$ to be a diagonal matrix whose $i$-th component corresponds to the sample variance of the $i$-th component of $\mathcal{G}(x)$ over a long time-window. In practice, the forward map $\mathcal{G}$ often does not have a closed mathematical expression, and it can be very expensive to evaluate as it requires solving differential equations numerically. In the following subsections, we consider Rossler and Lorenz dynamics and demonstrate the approximate Bayesian inference procedures based on EXPLOIT+ and GP-UCB+ surrogates.

\subsubsection{Rossler Dynamics}
Consider a set of ordinary differential equations describing the Rossler system given by
\begin{alignat*}{3}
    \frac{dz_1}{dt} &= -z_2-z_3, \qquad  &&z_1(0) = 1,\\
    \frac{dz_2}{dt} &= z_1 + 0.2 z_2, \qquad &&z_2(0) = 0, \\
    \frac{dz_3}{dt} &= 0.2 + z_3(z_1-x), \qquad &&z_3(0) = 1,
\end{alignat*}
where the true parameter was set to be $x^* = 5.7$. We assume data is the noise-corrupted image of the forward map, given by the first and second moments, which can be written as
\begin{align*}
\mathcal{D} = \frac{1}{T}\left(\int_{0}^T z_*(t)dt,  \int_{0}^T z_*^{2}(t)dt  \right) + \eta = \mathcal{G}(x^*) + \eta, \quad \eta \sim \mathcal{N}(0, \Gamma),
\end{align*}
where $z_*(t) = \bigl(z_{*,1}(t), z_{*,2}(t), z_{*,3}(t)\bigr)$ represents the solution path of the Rossler system with the true parameter $x^*$ and  $z_*^2(t) = (z_{*,1}^2(t), z_{*,2}^2(t), z_{*,3}^2(t), z_{*,1}(t)z_{*,2}(t), z_{*,1}(t)z_{*,3}(t), z_{*,2}(t)z_{*,3}(t))$ is the second-order product terms of $z_*$. For the noise covariance matrix $\Gamma$, we set it to be a diagonal matrix, whose components are sample variances of $(z(t), z^2(t))$ over the time-window $[20, 500]$. 

To perform the Bayesian inference, we assume that the prior of $x^*$ is $\mathcal{N}(6, 2^2)$. Our goal is then to build an accurate Gaussian process surrogate for the unnormalized log-posterior density given by
$$
 V^{\text{ROS}}(x) =  -\frac{1}{2}\|\mathcal{D}-\mathcal{G}(x)\|_{\Gamma}^2 - \frac{1}{8}(x - 6)^2,
$$ 
where $\mathcal{G}$ maps $x$ to the first and second-order moments of the corresponding dynamics. In numerical experiments, we approximated the continuum moments with finite averages. The time window we consider for the dynamics is set to be between $[20, 50]$, and the parameter space of interest is an interval $[1, 14]$, outside which the prior distribution has negligible mass. We run 20 independent experiments where each experiment involves 20 forward map evaluations for each of the four strategies: 1) GP-UCB, 2) Uniform random sample, 3) EXPLOIT+, and 4) GP-UCB+.

In Figure \ref{EXAMPLE_ROSSLER}, we provide plots for $V^{\text{ROS}}$ with its surrogates formed with design points obtained via the aforementioned four strategies, as well as the histogram of selected design locations for each of the four strategies we consider. As one can see from the histogram in Figure \ref{EXAMPLE_ROSSLER}, points selected by GP-UCB are highly concentrated around the posterior mode, while uniform random samples are evenly distributed along the search space. The two proposed algorithms, EXPLOIT+ and GP-UCB+, explored a wider range of parameter values in comparison to GP-UCB. In addition, the query locations of the proposed algorithms were more concentrated around the optimum compared to a uniform random sample. The patterns that appeared in the histogram are also reflected in the two surrogate plots in Figure \ref{EXAMPLE_ROSSLER}. The surrogate based on GP-UCB design points colored in blue approximates poorly  $V^{\text{ROS}}$ in the region to the right of the posterior mode, while it accurately approximates it in the interval (7, 8) around the posterior mode. On the other hand, the surrogate based on uniform random samples approximates $V^{\text{ROS}}$ better in the region to the right of the posterior mode, but the surrogate did not reach its maximum close to the posterior mode. Unlike the two aforementioned strategies, thanks to the balance between exploitation and exploration maintained by EXPLOIT+ and GP-UCB+, their surrogates reached the maximum close to the posterior mode, while also accurately reflecting the global shape of the energy function $V^{\text{ROS}}$.

\begin{figure}[htp]
    \centering
\includegraphics[height=.235\textwidth]{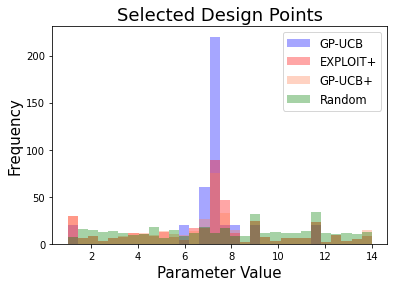}
\includegraphics[height=.235\textwidth]{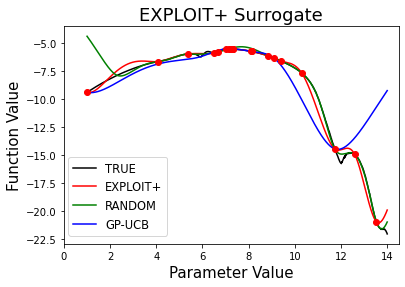}
\includegraphics[height=.235\textwidth]{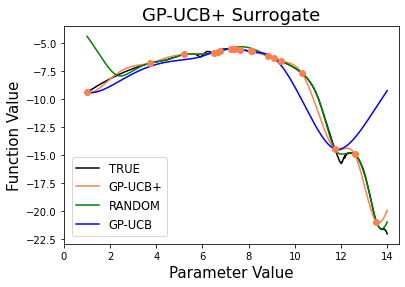}
\caption{Design points and surrogates for parameter inference in the Rossler system.}
\label{EXAMPLE_ROSSLER}
\end{figure}

Next, we implement a rejection sampling algorithm to obtain posterior samples from the surrogate distribution. As we have access to an accurate estimate of the maximum of the unnormalized log-posterior density, rejection sampling naturally serves as a parallelizable sampling algorithm that can generate samples from the surrogate posterior density. In Figure \ref{EXAMPLE_ROSSLER_REJ}, we provide histograms of 2000 samples with the true posterior density colored in black. For each of the aforementioned four strategies, the kernel density estimator of the obtained samples is plotted on top of the histogram. As one can see from Figure \ref{EXAMPLE_ROSSLER_REJ}, the lack of exploration due to the concentration of GP-UCB points in a neighborhood around the posterior mode makes the density based on GP-UCB surrogate models to be excessively concentrated around its mode, yielding an inaccurate approximation of the posterior density. In contrast, the density based on a random sampling surrogate model are more dispersed over the search space, resulting in poor posterior approximation near the true posterior mode. In comparison, surrogates based on both EXPLOIT+ and GP-UCB+ led to much more accurate approximations of the true posterior distribution. The surrogates not only captured the local shape of the unnormalized log-posterior density in the region around the posterior mode, but also depicted the overall shape of target density over the parameter space. For twenty independent experiments, we computed $\ell_2$-difference between the target and surrogate density vectors along the grid of size $1401$ in $[1,14]$. Their average and standard deviations are shown in Table \ref{rossler-table}, which clearly demonstrates the superior performance of the proposed algorithms in the posterior approximation task.

\begin{figure}[htp]
    \centering
\includegraphics[height=.25\textwidth]{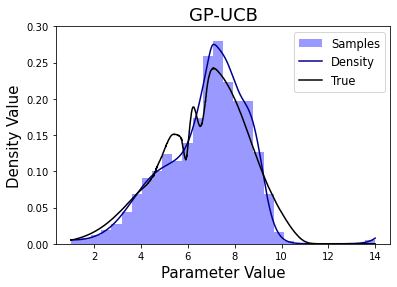}
\includegraphics[height=.25\textwidth]{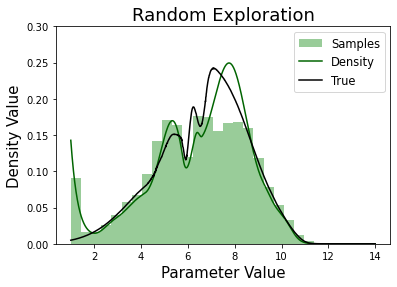}
\includegraphics[height=.25\textwidth]{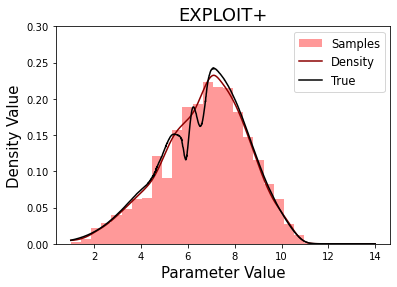}
\includegraphics[height=.25\textwidth]{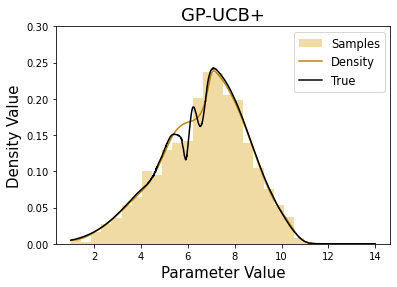}
\caption{Rejection sampling for posterior inference of a parameter in the Rossler system.}
\label{EXAMPLE_ROSSLER_REJ}
\end{figure}

\begin{table}[htp]
\begin{center}
\begin{sc}
\begin{tabular}{lcr}
\toprule
Method & $\ell_2$-difference $\pm$ SD  \\
\midrule
GP-UCB & $0.7134 \pm 0.0004$  \\
RANDOM & $1.1129 \pm 0.9910$   \\
EXPLOIT+  & $0.4285 \pm 0.0986$  \\
GP-UCB+  & $\mathbf{0.3569 \pm 0.0710} $   \\
\bottomrule
\end{tabular}
\end{sc}
\end{center}
\vskip -0.1in
\caption{\label{table:rossler_post_approx} Average $\ell_2$-difference between the true and approximate posterior densities.}
\label{rossler-table}
\end{table}

\subsubsection{Lorenz-63 Dynamics}
Let us now consider a set of ordinary differential equations describing the Lorenz-63 dynamical system given by
\begin{alignat*}{3}
    \frac{dz_1}{dt} &= x_1^*(z_2-z_1), \qquad \qquad &&z_1(0) = 1,\\
    \frac{dz_2}{dt} &= x_2^*z_1 - z_2 - z_1z_3, \quad &&z_2(0) = 0,\\
    \frac{dz_3}{dt} &= z_1z_2 - x_3^*z_3, \quad &&z_3(0) = 1,
\end{alignat*}
with true parameter $x^* = (x_1^*, x_2^*, x_3^*) = (10, 28, 8/3)$. Like in the Rossler dynamics example, we assume the data is the noise-corrupted image of the true parameter $x^*$ under the forward map $\mathcal{G}$, given by the first and second moments of the dynamics. We assume the prior of $x^*$ is $\mathcal{N}(m_0, P)$ with $m_0 = (10, 28.5, 2.7)$ and $P = \text{diag}([0.25, 2.25, 0.49])$. We consider a time-window $[10, 200]$, which can be computationally burdensome if one wants to obtain many forward map evaluations. Our goal is to build a Gaussian process surrogate for the unnormalized log-posterior density given by
$$
 V^{\text{LZ63}}(x) =  -\frac{1}{2}\|\mathcal{D}-\mathcal{G}(x)\|_{\Gamma}^2 - \frac{1}{2}\|x - m_0\|_P^2.
$$ 
For the noise covariance matrix $\Gamma$, we set it to be a diagonal matrix, whose components are sample variances of $(z(t), z^2(t))$ over the long time-window $[10, 2000]$ scaled by the noise level of $0.25$. The search space of consideration is set to be $[8.72, 11.28] \times [24.66,32.34] \times [0.908,4.492]$, which effectively contains 99\% of the mass of the prior distribution. We run ten independent experiments where each experiment involves 400 forward map evaluations (where for each of the four strategies: 1) GP-UCB, 2) Uniform random sample, 3) EXPLOIT+, and 4) GP-UCB+.

To compare the posterior approximations for each of the aforementioned four strategies, we implement a random walk Metropolis-Hastings algorithm and provide MCMC samples for each surrogate posterior. We ran 20000 MCMC iterations and discarded samples from the initial 10000 iterations as burn-in period samples. We tuned the parameter for random walk Metropolis-Hasting to roughly yield a 50\% acceptance rate. In Figure \ref{EXAMPLE_LORENZ}, we provide MCMC samples for each strategy along with the MCMC samples for the true posterior distribution. As one can see from the third row of Figure \ref{EXAMPLE_LORENZ}, MCMC samples based on GP-UCB strategies tend to concentrate around the posterior mode, but due to lack of exploration, the posterior approximation is very inaccurate far from the posterior mode. In the meantime, in the second row of Figure \ref{EXAMPLE_LORENZ}, for both the first and third parameters, the MCMC samples based on a uniform random sampling strategy concentrate far away from the mode of the true posterior density, indicating lack of exploitation. In comparison to GP-UCB or uniform random sampling, in both the first and fourth rows of Figure \ref{EXAMPLE_LORENZ}, MCMC samples based on EXPLOIT+ or GP-UCB+ closely capture the global shape of the target posterior while maintaining high accuracy around the true posterior mode. We computed the $\ell_2$-difference between the true unnormalized posterior density vector and that obtained from surrogate models based on four experimental design strategies: 1) GP-UCB, 2) Uniform random sample, 3) EXPLOIT+, and 4) GP-UCB+ along the Latin hypercube samples of size 30000. We computed the differences for ten independent experiments and provided their averages in Table \ref{table:Lorenz_post_approx}. As one can see from Table \ref{table:Lorenz_post_approx}, surrogate posterior densities obtained from EXPLOIT+ and GP-UCB+ had a considerably smaller $\ell_2$-difference in comparison to vanilla GP-UCB and uniform random exploration. In particular, we observed that even for a moderate parameter dimension, Gaussian process surrogates based on random sampled design points can be highly inaccurate unless a significant number of design points are used. On the other hand, the design points selected by GP-UCB algorithm yielded a somewhat better Gaussian process surrogate as it effectively captured the local geometry of the objective function in the region around the mode. Appropriately combining design points from random exploration and GP-UCB, both EXPLOIT+ and GP-UCB+ facilitated building more accurate surrogate models with no additional computational cost.  \nc

\begin{figure}[htp]
    \centering
\includegraphics[height=.235\textwidth]{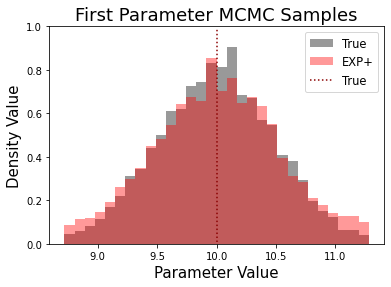}
\includegraphics[height=.235\textwidth]{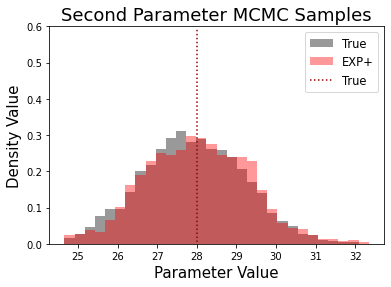}
\includegraphics[height=.235\textwidth]{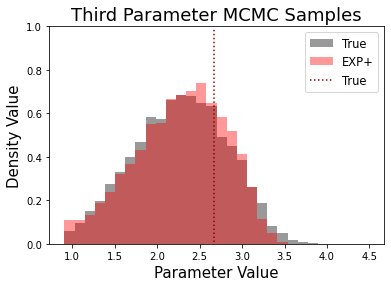}
\includegraphics[height=.235\textwidth]{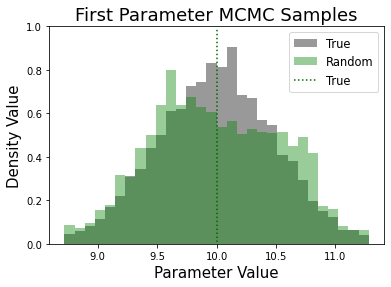}
\includegraphics[height=.235\textwidth]{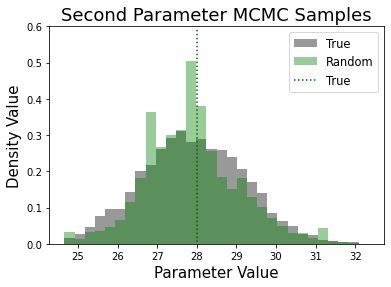}
\includegraphics[height=.235\textwidth]{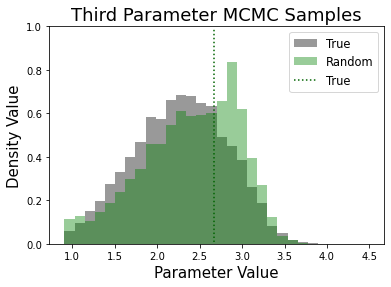}
\includegraphics[height=.235\textwidth]{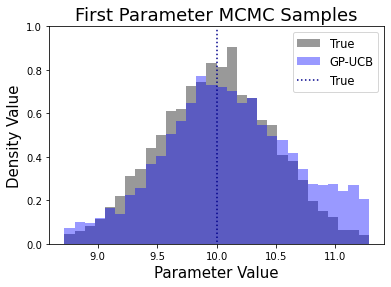}
\includegraphics[height=.235\textwidth]{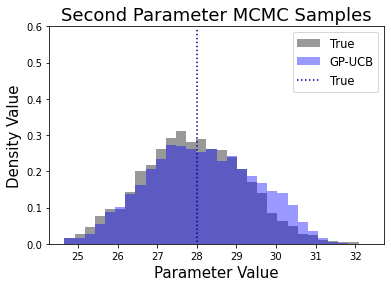}
\includegraphics[height=.235\textwidth]{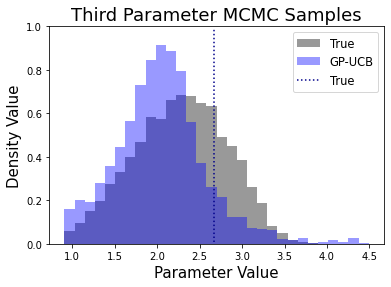}
\includegraphics[height=.235\textwidth]{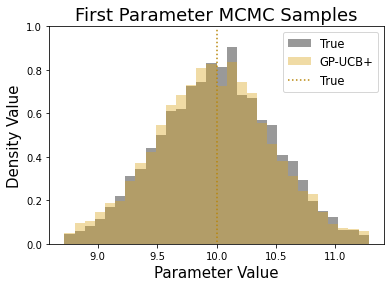}
\includegraphics[height=.235\textwidth]{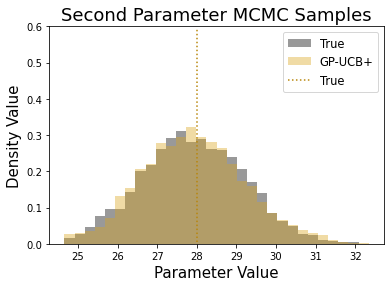}
\includegraphics[height=.235\textwidth]{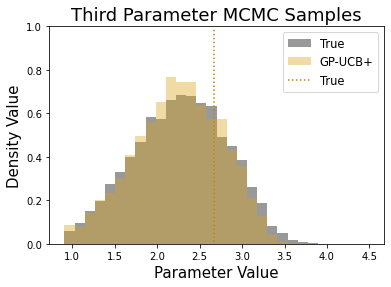}
\caption{MCMC samples for posterior inference of parameters in the Lorenz-63 system.}
\label{EXAMPLE_LORENZ}
\end{figure}

\begin{table}[htp]
\begin{center}
\begin{sc}
\begin{tabular}{lcr}
\toprule
Method & $\ell_2$-difference $\pm$ SD  \\
\midrule
EXPLOIT+  & $14.1004 \pm 1.3429$  \\
Random & $80.6782 \pm 54.2686$   \\
GP-UCB & $25.0256 \pm 0.7344$  \\
GP-UCB+  & $\mathbf{11.6711 \pm 1.4865} $   \\
\bottomrule
\end{tabular}
\end{sc}
\end{center}
\caption{\label{table:Lorenz_post_approx} Average $\ell_2$-difference between the true and approximate posterior densities.}
\label{Lorenz-table}
\end{table}

\section{Conclusion}\label{sec:conclusions}
This paper has introduced two Bayesian optimization algorithms, GP-UCB+ and EXPLOIT+, that supplement query points obtained via UCB or posterior mean maximization with query points obtained via random sampling. The additional sampling step in our algorithms promotes search space exploration and ensures that the fill-distance of the query points decays at a nearly optimal rate. From a theoretical viewpoint, we have shown that GP-UCB+ and EXPLOIT+ achieve a near-optimal convergence rate that improves upon existing and refined rates for the classical GP-UCB algorithm with noise-free observations.  Indeed, at the price of a higher computational cost, one can obtain the optimal convergence rate from \cite{bull2011convergence} as well as the convergence rate implied by the conjectured cumulative regret bound in \cite{vakili2022open} by replacing the random sampling step in GP-UCB+ and EXPLOIT+ with a quasi-uniform sampling scheme. From an implementation viewpoint, both GP-UCB+ and EXPLOIT+ retain the appealing simplicity of the GP-UCB algorithm; moreover, EXPLOIT+ does not require specifying input weight parameters. From an empirical viewpoint, we have demonstrated that the new algorithms outperform existing ones in a wide range of examples. This paper further proposes to use our new Bayesian optimization strategies as an experimental design tool to build surrogate posterior distributions for Bayesian inference tasks when the likelihood is intractable. The resulting posterior approximations are provably accurate
and perform well in several computed examples.


\bibliographystyle{siam}
\bibliography{references}

\begin{thebibliography}{10}

\bibitem{bansal2017goal}
{\sc S.~Bansal, R.~Calandra, T.~Xiao, S.~Levine, and C.~J. Tomlin}, {\em {Goal-driven dynamics learning via Bayesian optimization}}, in 2017 IEEE 56th Annual Conference on Decision and Control (CDC), IEEE, 2017, pp.~5168--5173.

\bibitem{bogunovic2021misspecified}
{\sc I.~Bogunovic and A.~Krause}, {\em {Misspecified Gaussian process bandit optimization}}, Advances in Neural Information Processing Systems, 34 (2021), pp.~3004--3015.

\bibitem{bull2011convergence}
{\sc A.~D. Bull}, {\em Convergence rates of efficient global optimization algorithms}, Journal of Machine Learning Research, 12 (2011).

\bibitem{burges1996improving}
{\sc C.~J. Burges and B.~Sch{\"o}lkopf}, {\em Improving the accuracy and speed of support vector machines}, Advances in Neural Information Processing Systems, 9 (1996).

\bibitem{chowdhury2017kernelized}
{\sc S.~R. Chowdhury and A.~Gopalan}, {\em On kernelized multi-armed bandits}, in International Conference on Machine Learning, PMLR, 2017, pp.~844--853.

\bibitem{ray2019bayesian}
\leavevmode\vrule height 2pt depth -1.6pt width 23pt, {\em Bayesian optimization under heavy-tailed payoffs}, Advances in Neural Information Processing Systems, 32 (2019).

\bibitem{clark2016engineering}
{\sc D.~L. Clark~Jr, H.-R. Bae, K.~Gobal, and R.~Penmetsa}, {\em Engineering design exploration using locally optimized covariance kriging}, AIAA Journal, 54 (2016), pp.~3160--3175.

\bibitem{cleary2021calibrate}
{\sc E.~Cleary, A.~Garbuno-Inigo, S.~Lan, T.~Schneider, and A.~M. Stuart}, {\em Calibrate, emulate, sample}, Journal of Computational Physics, 424 (2021), p.~109716.

\bibitem{de2012exponential}
{\sc N.~De~Freitas, A.~J. Smola, and M.~Zoghi}, {\em {Exponential regret bounds for Gaussian process bandits with deterministic observations}}, in Proceedings of the 29th International Coference on International Conference on Machine Learning, 2012, pp.~955--962.

\bibitem{frazier2018tutorial}
{\sc P.~I. Frazier}, {\em {A tutorial on Bayesian optimization}}, arXiv preprint arXiv:1807.02811,  (2018).

\bibitem{helin2023introduction}
{\sc T.~Helin, A.~M. Stuart, A.~L. Teckentrup, and K.~Zygalakis}, {\em {Introduction to Gaussian process regression in Bayesian inverse problems, with new results on experimental design for weighted error measures}}, arXiv preprint arXiv:2302.04518,  (2023).

\bibitem{jones1998efficient}
{\sc D.~R. Jones, M.~Schonlau, and W.~J. Welch}, {\em Efficient global optimization of expensive black-box functions}, Journal of Global optimization, 13 (1998), p.~455.

\bibitem{kanagawa2019convergence}
{\sc M.~Kanagawa and P.~Hennig}, {\em Convergence guarantees for adaptive bayesian quadrature methods}, Advances in neural information processing systems, 32 (2019).

\bibitem{kanagawa2018gaussian}
{\sc M.~Kanagawa, P.~Hennig, D.~Sejdinovic, and B.~K. Sriperumbudur}, {\em Gaussian processes and kernel methods: A review on connections and equivalences}, arXiv preprint arXiv:1807.02582,  (2018).

\bibitem{kandasamy2018parallelised}
{\sc K.~Kandasamy, A.~Krishnamurthy, J.~Schneider, and B.~P{\'o}czos}, {\em {Parallelised Bayesian optimisation via Thompson sampling}}, in International Conference on Artificial Intelligence and Statistics, PMLR, 2018, pp.~133--142.

\bibitem{kim2022optimization}
{\sc H.~Kim, D.~Sanz-Alonso, and R.~Yang}, {\em {Optimization on manifolds via graph Gaussian processes}}, arXiv preprint arXiv:2210.10962,  (2022).

\bibitem{lyu2019efficient}
{\sc Y.~Lyu, Y.~Yuan, and I.~W. Tsang}, {\em Efficient batch black-box optimization with deterministic regret bounds}, arXiv preprint arXiv:1905.10041,  (2019).

\bibitem{mockus1998application}
{\sc J.~Mockus}, {\em {The application of Bayesian methods for seeking the extremum}}, Towards Global Optimization, 2 (1998), p.~117.

\bibitem{novak2006deterministic}
{\sc E.~Novak}, {\em Deterministic and stochastic error bounds in numerical analysis}, vol.~1349, Springer, 2006.

\bibitem{oates2019convergence}
{\sc C.~J. Oates, J.~Cockayne, F.-X. Briol, and M.~Girolami}, {\em {Convergence rates for a class of estimators based on Stein’s method}}, Bernoulli, 25 (2019), pp.~1141--1159.

\bibitem{pace1997sparse}
{\sc R.~K. Pace and R.~Barry}, {\em Sparse spatial autoregressions}, Statistics \& Probability Letters, 33 (1997), pp.~291--297.

\bibitem{pourmohamad2020compmodels}
{\sc T.~Pourmohamad}, {\em Compmodels: Pseudo computer models for optimization}, R package version 0.2. 0,  (2020).

\bibitem{pourmohamad2021bayesian}
{\sc T.~Pourmohamad and H.~K. Lee}, {\em Bayesian Optimization with Application to Computer Experiments}, Springer, 2021.

\bibitem{russo2014learning}
{\sc D.~Russo and B.~Van~Roy}, {\em Learning to optimize via information-directed sampling}, Advances in Neural Information Processing Systems, 27 (2014).

\bibitem{salgia2023random}
{\sc S.~Salgia, S.~Vakili, and Q.~Zhao}, {\em {Random exploration in Bayesian Optimization: order-optimal regret and computational efficiency}}, arXiv preprint arXiv:2310.15351,  (2023).

\bibitem{schneider2017earth}
{\sc T.~Schneider, S.~Lan, A.~Stuart, and J.~Teixeira}, {\em Earth system modeling 2.0: A blueprint for models that learn from observations and targeted high-resolution simulations}, Geophysical Research Letters, 44 (2017), pp.~12396--12417.

\bibitem{siebertz2010statistische}
{\sc K.~Siebertz, T.~Hochkirchen, and D.~van Bebber}, {\em Statistische Versuchsplanung}, Springer, 2010.

\bibitem{singer2023alignment}
{\sc A.~Singer and R.~Yang}, {\em {Alignment of density maps in Wasserstein distance}}, Biological Imaging, 4 (2024), p.~e5.

\bibitem{srinivas2009gaussian}
{\sc N.~Srinivas, A.~Krause, S.~Kakade, and M.~Seeger}, {\em Gaussian process optimization in the bandit setting: no regret and experimental design}, in Proceedings of the 27th International Conference on Machine Learning, 2010, pp.~1015--1022.

\bibitem{stein2012interpolation}
{\sc M.~L. Stein}, {\em Interpolation of Spatial Data: Some Theory for Kriging}, Springer Science \& Business Media, 2012.

\bibitem{stuart2018posterior}
{\sc A.~Stuart and A.~Teckentrup}, {\em {Posterior consistency for Gaussian process approximations of Bayesian posterior distributions}}, Mathematics of Computation, 87 (2018), pp.~721--753.

\bibitem{teckentrup2020convergence}
{\sc A.~L. Teckentrup}, {\em {Convergence of Gaussian process regression with estimated hyper-parameters and applications in Bayesian inverse problems}}, SIAM/ASA Journal on Uncertainty Quantification, 8 (2020), pp.~1310--1337.

\bibitem{tuo2020kriging}
{\sc R.~Tuo and W.~Wang}, {\em {Kriging prediction with isotropic Mat{\'e}rn correlations: Robustness and experimental designs}}, Journal of Machine Learning Research, 21 (2020), pp.~7604--7641.

\bibitem{vakili2022open}
{\sc S.~Vakili}, {\em Open problem: Regret bounds for noise-free kernel-based bandits}, in Conference on Learning Theory, PMLR, 2022, pp.~5624--5629.

\bibitem{vakili2021information}
{\sc S.~Vakili, K.~Khezeli, and V.~Picheny}, {\em {On information gain and regret bounds in Gaussian process bandits}}, in International Conference on Artificial Intelligence and Statistics, PMLR, 2021, pp.~82--90.

\bibitem{wang2020prediction}
{\sc W.~Wang, R.~Tuo, and C.~Jeff~Wu}, {\em {On prediction properties of kriging: Uniform error bounds and robustness}}, Journal of the American Statistical Association, 115 (2020), pp.~920--930.

\bibitem{wendland2004scattered}
{\sc H.~Wendland}, {\em Scattered Data Approximation}, vol.~17, Cambridge University Press, 2004.

\bibitem{wenzel2021novel}
{\sc T.~Wenzel, G.~Santin, and B.~Haasdonk}, {\em A novel class of stabilized greedy kernel approximation algorithms: Convergence, stability and uniform point distribution}, Journal of Approximation Theory, 262 (2021), p.~105508.

\bibitem{williams2006gaussian}
{\sc C.~K. Williams and C.~E. Rasmussen}, {\em Gaussian Processes for Machine Learning}, vol.~2, MIT Press Cambridge, MA, 2006.

\bibitem{wu1993local}
{\sc Z.-m. Wu and R.~Schaback}, {\em Local error estimates for radial basis function interpolation of scattered data}, IMA Journal of Numerical Analysis, 13 (1993), pp.~13--27.

\end{thebibliography}

\appendix

\section{Proofs}

\begin{proof}[Proof of Theorem \ref{UCB_cum_reg_bd}]
Let $f^* = f(x^*) = \max_{x \in \mathcal{X}} f(x)$ and let $r_t = f^* - f(x_t)$ be the instantaneous regret. Then, 
\begin{align}
\begin{split}\label{eq:intantaneousboundGPUCB}
  r_t &= f^* - \mu_{t-1,0}(x^*) + \mu_{t-1,0}(x^*) - \mu_{t-1,0}(x_t) + \mu_{t-1,0}(x_t)- f(x_t)  \\
  &\overset{\text{(i)}}{\le} \|f\|_{\mathcal{H}_k(\mathcal{X})}\sigma_{t-1,0}(x^*) + \mu_{t-1,0}(x^*) - \mu_{t-1,0}(x_t) 
 + \|f\|_{\mathcal{H}_k(\mathcal{X})}\sigma_{t-1,0}(x_t)  \\
  & \overset{\text{(ii)}}{\le} \|f\|_{\mathcal{H}_k(\mathcal{X})}\sigma_{t-1,0}(x_t) + \mu_{t-1,0}(x_t) - \mu_{t-1,0}(x_t) +\|f\|_{\mathcal{H}_k(\mathcal{X})}\sigma_{t-1,0}(x_t) \\
  & = 2\|f\|_{\mathcal{H}_k(\mathcal{X})} \sigma_{t-1,0}(x_t),
  \end{split}
\end{align}
where for (i) we use twice that, for any $x \in \mathcal{X},$ it holds that $|f(x) - \mu_{t-1,0}(x)| \le \|f\|_{\mathcal{H}_k(\mathcal{X})}\sigma_{t-1,0}(x)$
---see for instance Corollary 3.11 in \citep{kanagawa2018gaussian}---
 and for (ii) we use the definition of $x_t$ in the GP-UCB algorithm. Thus, for any $\lambda >0,$ 
\begin{align*}
   R_T^2 &\overset{\text{(i)}}{\le} T \sum_{t=1}^T r_t^2
   \overset{\text{(ii)}}{\le} 4 T \|f\|_{\mathcal{H}_k(\mathcal{X})}^2  \sum_{t=1}^T \sigma^2_{t-1,0}(x_t) 
   \overset{\text{(iii)}}{\le} 4 T \|f\|_{\mathcal{H}_k(\mathcal{X})}^2  \sum_{t=1}^T \sigma^2_{t-1,\lambda}(x_t), 
\end{align*}
where (i) follows by the Cauchy-Schwarz inequality, (ii) from the bound on $r_t$, and (iii) from the fact that  $\sigma_{t-1,0}(x_t) \le \sigma_{t-1,\lambda}(x_t)$ for any $\lambda > 0$. Since the function $\frac{x}{\log(1+x)}$ is strictly increasing in $x$ and for the squared exponential kernel it holds that $\lambda^{-1}\sigma^2_{t-1,\lambda}(x_t) \le \lambda^{-1},$ we have that $\lambda^{-1}\sigma^2_{t-1,\lambda}(x_t) \le \frac{\lambda^{-1}}{\log(1+\lambda^{-1})}\log \bigl(1+\lambda^{-1}\sigma^2_{t-1,\lambda}(x_t)\bigr).$  Therefore, 
\begin{equation}\label{cum_reg_bound}
R_T^2 \le   \frac{8 T \|f\|_{\mathcal{H}_k(\mathcal{X})}^2 }{\log(1+\lambda^{-1})} \left( \frac{1}{2}\sum_{t=1}^T \log\Bigl(1 + \lambda^{-1}\sigma^2_{t-1, \lambda}(x_t)\Bigr)\right) \le \frac{8 T \|f\|_{\mathcal{H}_k(\mathcal{X})}^2 }{\log(1+\lambda^{-1})} \gamma_{T,\lambda},
\end{equation}
where the last inequality follows from  Lemma 5.3 in \cite{srinivas2009gaussian}. Since \eqref{cum_reg_bound} holds for any $\lambda > 0$, by plugging $\lambda = T^{-\alpha}$, for some $\alpha > 0$, we conclude that 
\begin{equation}\label{cum_reg_tight_bound}
R_T^2 \le \frac{8T \|f\|_{\mathcal{H}_k(\mathcal{X})}^2}{\log(1+T^{\alpha})} \gamma_{T,T^{-\alpha}}.
\end{equation} 

 For squared exponential kernels, Corollary 1 in \cite{vakili2021information} implies that
 \begin{align*}
\gamma_{T,T^{-\alpha}} &\le \left(\left(2(1+\alpha)\log T + \tilde C(d) \right)^d + 1 \right)\log\left(1 + T^{1+\alpha} \right) \lesssim \log^d (T) \log\left(1+T^{1+\alpha}\right),
\end{align*}
where $\tilde C(d) = \mathcal{O}(d \log d)$ is independent of $T$ and $\lambda.$
Hence, using that $\frac{\log(1+T^{\alpha+1})}{\log(1+T^\alpha)} \le \frac{\alpha+1}{\alpha}$ for $\alpha > 0, T \ge 1$, we obtain
\begin{align*}
R_T^2  \lesssim T \log^d (T) \frac{\log(1+T^{1+\alpha})}{\log(1+T^\alpha)} \lesssim T \log^d (T),
\end{align*}
concluding the proof. 
\end{proof}

\begin{proof}[Proof of Theorem \ref{thm:CumulativeRegretBounds}]
We first prove the cumulative regret bound for GP-UCB+. As in \eqref{eq:intantaneousboundGPUCB}, one can show that 
\begin{align*}
   r_t \le 2\|f\|_{\mathcal{H}_k(\mathcal{X})}\sigma^{\text{full}}_{t-1,0}(x_T) \le 2\|f\|_{\mathcal{H}_k(\mathcal{X})} \sup_{x \in \mathcal{X}}\sigma^{\text{full}}_{t-1,0}(x).
 \end{align*}
For Matérn kernels, for large $t \in \mathbb{N}$, \cite{wu1993local} shows that $\sup_{x \in \mathcal{X}}\sigma^{\text{full}}_{t-1,0}(x) \le h(\mathcal{X},  X^{\text{full}}_t)^\nu$ ---see also Lemma 2 in \cite{wang2020prediction}. Moreover, we have the trivial bound 
$$
h(\mathcal{X},  X^{\text{full}}_t) \le h_t(\mathcal{X}) \coloneqq \sup_{x \in \mathcal{X}} \inf_{ \tilde x_i \in \{\tilde x_1, \ldots, \tilde x_t\}} \|x- \tilde x_i\|.
$$ 
Hence, for any $\epsilon >0,$
\begin{align*}
    \mathbb{E}_P[r_t] \lesssim  \mathbb{E}_P \Bigl[ \sup_{x \in \mathcal{X}} \sigma^{\text{full}}_{t-1,0}(x) \Bigr] 
     \lesssim  \mathbb{E}_P \Bigl[ h_t(\mathcal{X})^\nu \Bigr]
     \overset{(\star)}{\lesssim} t^{- \frac{\nu}{d} + \varepsilon},
\end{align*}
where $(\star)$ follows from Proposition 4 in \cite{helin2023introduction} ---see also Lemma 2 in \cite{oates2019convergence}.

For squared exponential kernels, Theorem 11.22 in \cite{wendland2004scattered} shows that, for some $C>0,$ $\sup_{x \in \mathcal{X}}\sigma^{\text{full}}_{t-1,0}(x) \le \exp\bigl( - C / h(\mathcal{X},  X^{\text{full}}_t) \bigr).$ Hence, for any  $\epsilon \le \frac{1}{2d}$, 
\begin{align*}
       \mathbb{E}_P[r_t] 
    \lesssim \mathbb{E}_P \Bigl[ \sup_{x \in \mathcal{X}}\sigma^{\text{full}}_{t-1,0}(x) \Bigr] 
    \lesssim  \mathbb{E}_P \Bigl[\exp\bigl(-C/h_t(\mathcal{X})\bigr)\Bigr]
    \overset{(\star)}{\lesssim} \exp\left(-C t^{\frac{1}{d}-\varepsilon}\right),
\end{align*}
where $(\star)$ follows from Proposition 4 in \cite{helin2023introduction} ---see also Lemma 2 in \cite{oates2019convergence}.

For the EXPLOIT+ algorithm, we have that
\begin{align*}
   r_t & = f^* - f(x_t) \\
   &= f^* - \mu_{t-1,0}^{\text{full}}(x^*) + \mu_{t-1,0}^{\text{full}}(x^*) - \mu_{t-1,0}^{\text{full}}(x_t) + \mu_{t-1,0}^{\text{full}}(x_t) - f(x_t)  \\
    &\overset{\text{(i)}}{\le} \| f\|_{\mathcal{H}_k} \sigma_{t-1,0}^{\text{full}}(x^*) + \mu_{t-1,0}^{\text{full}}(x^*) - \mu_{t-1,0}^{\text{full}}(x_t) + \| f\|_{\mathcal{H}_k} \sigma_{t-1,0}^{\text{full}}(x_t) \\ 
   &\overset{\text{(ii)}}{\le} \|f\|_{\mathcal{H}_k(\mathcal{X})}\sigma^{\text{full}}_{t-1,0}(x^*) + \|f\|_{\mathcal{H}_k(\mathcal{X})}\sigma^{\text{full}}_{t-1,0}(x_t) \\
   &\le 2\|f\|_{\mathcal{H}_k(\mathcal{X})} \sup_{x \in \mathcal{X}}\sigma^{\text{full}}_{t-1,0}(x),
 \end{align*}
where for (i) we use twice that, for any $x \in \mathcal{X},$ it holds that $|f(x) - \mu_{t-1,0}(x)| \le \|f\|_{\mathcal{H}_k(\mathcal{X})}\sigma_{t-1,0}(x),$
 and for (ii) we use the definition of $x_t$ in the EXPLOIT+ algorithm. The rest of the proof proceeds exactly as the one for GP-UCB+, and we hence omit the details.
\end{proof}

\begin{proof}[Proof of Proposition \ref{prop:norm_const_bound}]
The bounds follow from the continuity of $V$ and the compactness of $\mathcal{X}$. Specifically, we have
\begin{align*}
   Z_{\pi, t} &= \int_{\mathcal{X}}\exp\left(\mu_t^V(x)\right)dx \le \text{vol}(\mathcal{X}) \sup_{x \in \mathcal{X}} \exp \left( \mu_t^V(x) \right) \lesssim \sup_{x \in \mathcal{X}} \exp \left(V(x) + \|V\|_{\mathcal{H}_k}\right) \lesssim \sup_{x \in \mathcal{X}}\exp \left(V(x) \right)< \infty, \\
    Z_{\pi, t} &= \int_{\mathcal{X}}\exp\left(\mu_t^V(x)\right)dx \ge \text{vol}(\mathcal{X}) \inf_{x \in \mathcal{X}} \exp \left( \mu_t^V(x) \right) \gtrsim \inf_{x \in \mathcal{X}} \exp \left( V(x)- \|V\|_{\mathcal{H}_k}\right) \gtrsim \inf_{x \in \mathcal{X}} \exp \left( V(x)\right) > 0,
\end{align*}
where we have used Corollary 3.11 in \cite{kanagawa2018gaussian} and the fact that $|k(x,x)| \le 1$ for Mat\'ern and squared exponential kernels.
\end{proof}

\begin{proof}[Proof of Theorem \ref{thm:pos_approx_log_unnorm}]
Notice that, for $\sqrt{\tilde \pi_t(x)} \coloneqq \frac{1}{Z_{\pi, t}}\exp\left(V(x) \right)$
\begin{align*}
    2d^2_{H}(\pi, \pi_t) &= \int_{\mathcal{X}}\left(\sqrt{\pi(x)}-\sqrt{\pi_t(x)}\right)^2 dx \\
     &= \int_{\mathcal{X}}\left(\sqrt{\pi(x)}-\sqrt{\tilde \pi_t(x)}+\sqrt{\tilde \pi_t(x)}-\sqrt{\pi_t(x)}\right)^2 dx \\
    &\le 2  \int_{\mathcal{X}}\left(\sqrt{\pi(x)}-\sqrt{\tilde \pi_t(x)}\right)^2 dx + 2\int_{\mathcal{X}}\left(\sqrt{\tilde\pi_t(x)}-\sqrt{\pi_t(x)}\right)^2 dx  \\ 
    & = \frac{2}{Z} \int_{\mathcal{X}} \left(\exp\left(\frac{V(x)}{2}\right) - \exp\left(\frac{\mu_t^V(x)}{2} \right) \right)^2 dx + 2 \left(\frac{1}{\sqrt{Z_\pi}} - \frac{1}{\sqrt{Z_{\pi, t}}} \right)^2 Z_{\pi, t},
\end{align*}
where in the third line we used the fact that $(a+b)^2 \le 2a^2 + 2b^2$. Note that for the first term, thanks to the continuity of $V, \mu_t^V$ and compactness of $\mathcal{X}$, we have the local Lipschitzness of the exponential function. Therefore, we have
$$
\frac{2}{Z} \int_{\mathcal{X}} \left(\exp\left(\frac{V(x)}{2}\right) - \exp\left(\frac{\mu_t^V(x)}{2} \right) \right)^2 dx
 \lesssim \int |V(x) - \mu_t^V(x)|^2 dx = \|V-\mu_t^V\|^2_{L^2(\mathcal{X})}.
$$
Next, using the mean value theorem on $g(x) = x^{-1/2}$ with boundedness of $Z_{\pi, t}$, we have
\begin{align*}
    \left(\frac{1}{\sqrt{Z_\pi}} - \frac{1}{\sqrt{Z_{\pi, t}}} \right)^2 &\lesssim \left|Z_\pi - Z_{\pi, t}\right|^2\\
    &= \left(\int_{\mathcal{X}} \exp\bigl(V(x)\bigr) - \exp\left(\mu^V_t(x)\right) dx \right)^2 \\
    &\le \int_{\mathcal{X}}  \left(\exp\bigl(V(x)\bigr) - \exp\left(\mu^V_t(x)\right)\right)^2 dx \\
    &\lesssim \|V - \mu_t^V\|^2_{L^2(\mathcal{X})},
\end{align*}
where we used Jensen's inequality in the third line and the fourth line is due to the local Lipschitzness. Combining both terms, we have $d_H(\pi, \pi_t) \lesssim \|V - \mu_t^V\|^2_{L^2(\mathcal{X})}$. 

For Mat\'ern kernels, from Proposition 3.3 of \cite{stuart2018posterior}, we have
$$
d_H(\pi, \pi_t) \lesssim \|V - \mu_t^V\|_{L^2(\mathcal{X})} \lesssim h\left(\mathcal{X}, \mathcal{X}_t^{\text{full}}\right)^{\nu+\frac{d}{2}}\lesssim h\left(\mathcal{X}\right)^{\nu+\frac{d}{2}},
$$
where $h_t(\mathcal{X}) \coloneqq \sup_{x \in \mathcal{X}} \inf_{ \tilde x_i \in \{\tilde x_1, \ldots, \tilde x_t\}} \|x- \tilde x_i\|$. Applying Proposition 4 in \cite{helin2023introduction} -- see also Lemma 2 in \cite{oates2019convergence}, we have
$$
\mathbb{E}_P \left[d_H(\pi, \pi_t) \right] \lesssim \mathbb{E}_P \left[h\left(\mathcal{X}\right)^{\nu+\frac{d}{2}} \right] \lesssim t^{-\frac{\nu}{d}-\frac{1}{2} + \epsilon},
$$
for arbitrarily small $\epsilon > 0$.

For squared exponential kernels, from Theorem 11.22 of \cite{wendland2004scattered}, we have
$$
d_H(\pi, \pi_t) \lesssim \|V - \mu_t^V\|_{L^2(\mathcal{X})} \lesssim \exp\left(-C/h\left(\mathcal{X}, \mathcal{X}_t^{\text{full}}\right) \right) \lesssim \exp\left(-C/h\left(\mathcal{X}\right) \right),
$$
for some constant $C > 0$. Applying Proposition 4 in \cite{helin2023introduction} -- see also Lemma 2 in \cite{oates2019convergence}, we have
$$
\mathbb{E}_P \left[d_H(\pi, \pi_t) \right] \lesssim \mathbb{E}_P \left[\exp\left(-C/h\left(\mathcal{X}\right) \right) \right] \lesssim \exp \left(-C t^{\frac{1}{d}-\epsilon}\right) ,
$$
for some constant $C > 0$, with an arbitrarily small $\epsilon > 0$.
\end{proof}

\section{Additional Experiments and Implementation Details: Benchmark Functions}
\label{appendix:BENCHMARK}
This appendix provides detailed descriptions of the numerical experiments conducted in Section \ref{subsec:BENCH}. The functional forms of the three objective functions we considered and their respective search space are provided below. For all three benchmark functions we denote  $x = (x^1, \ldots, x^d)$ and set $d = 10.$
\begin{itemize}
    \item Ackley function: 
    $$
        f(x) = -20 \exp\left(-\frac{1}{5}\sqrt{\frac{1}{d}\sum_{i=1}^d (x^i)^2} \right) - \exp\left(\frac{1}{d}\sum_{i=1}^d \cos(2\pi x^i) \right) + 20 + \exp(1), \quad x \in [-32.768, 32.768]^d.
    $$
    \item Rastrigin function:
    $$
        f(x) = 10d + \sum_{i=1}^d [(x^i)^2 - 10 \cos(2\pi x^i)], \quad x \in [-5.12, 5.12]^d.
    $$
    \item Levy function: With $\omega_i = 1 + \frac{x^i-1}{4}$, for all $i \in \{1, \ldots, d\}$ 
    $$
        f(x) = \sin^2(\pi \omega_1) + \sum_{i=1}^{d-1}(\omega_i -1)^2[1+10\sin^2(\pi\omega_i + 1)] + (\omega_d-1)^2[1+\sin^2(2\pi \omega_d)], \quad x \in [-10, 10]^d.
    $$
\end{itemize}

\begin{table}[H]
\vskip 0.15in
\begin{center}
\begin{sc}
\begin{tabular}{lcccr}
\toprule
Method & Ackley & Rastrigin & Levy  \\
\midrule

GP-UCB+ & \bf{0.075} & 0.797 & \bf{0.131} \\
GP-UCB & 1.000 & 1.000 & 0.719  \\
EXPLOIT+ & 0.306 & 0.577 & 0.127  \\
EXPLOIT & 0.733 & 0.976 & 1.000   \\
EI   & 0.466 & 0.609 & 0.160 \\
PI & 0.329 & \bf{0.360} & 0.312  \\
\bottomrule
\end{tabular}
\end{sc}
\end{center}
\caption{\label{table:benchmarksstd} Normalized average standard deviation of simple regret with 400 function evaluations for different benchmark objectives in dimension $d = 10.$}
\end{table}

Recall that Figure \ref{Benchmark_SIMPLE_PLOT} portrayed the average  simple regret of the six Bayesian optimization strategies we consider: GP-UCB+ (proposed algorithm), GP-UCB (\cite{srinivas2009gaussian}) (both with the choice of $\beta_t = 2$), EXPLOIT+ (proposed algorithm), EXPLOIT (GP-UCB with $\beta_t = 0$), EI (Expected Improvement), and PI (Probability of Improvement). The simple regret values at the last iteration were displayed in Table \ref{table:benchmarks}. Furthermore, Table \ref{table:benchmarksstd} shows the standard deviations of the last simple regret values over 20 independent experiments. From Figure \ref{Benchmark_SIMPLE_PLOT} and Table \ref{table:benchmarksstd}, one can see that not only were the proposed methods (GP-UCB+ and EXPLOIT+) able to yield superior simple regret performance, but also their standard deviations were substantially smaller than those of the other methods, indicating superior stability.

 Additionally, Figure \ref{Benchmark_beta_CUM_PLOT} shows the cumulative regret for GP-UCB+ and GP-UCB with different choices of $\beta_t$.  All results were averaged over 20 independent experiments. We considered $\beta_t^{1/2} = 2$ and $\beta_t^{1/2} = \max_{x \in \mathcal{X}_D} |f(x)|$ where $\mathcal{X}_D$ is a set of 100 Latin hypercube samples. In all experiments, $\max_{x \in \mathcal{X}_D} |f(x)|$ was significantly larger than $2$. Figure \ref{Benchmark_beta_CUM_PLOT} demonstrates that the choice of $\beta_t$ can significantly influence the cumulative regret. In particular, we have observed that the GP-UCB+ algorithm tends to work better with smaller $\beta_t$ values, as the algorithm contains additional exploration steps through random sampling; this behavior can also be seen in Figure \ref{Benchmark_beta_CUM_PLOT}.  In all three benchmark functions, GP-UCB exhibits sensitivity to the choice of parameter $\beta_t$; in contrast, our EXPLOIT+ algorithm does not require specifying weight parameters and consistently achieves competitive or improved performance across all our experiments. 

\begin{figure}[htp]
  \centering
\includegraphics[scale=0.345]{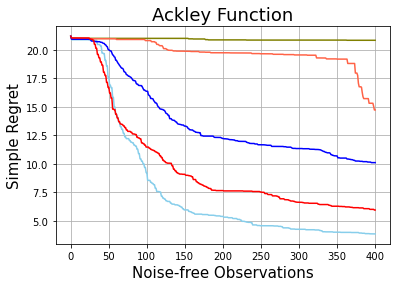}
\includegraphics[scale=0.345]{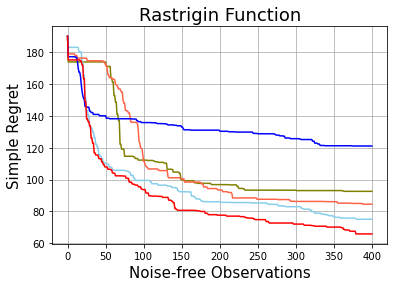}
\includegraphics[scale=0.345]{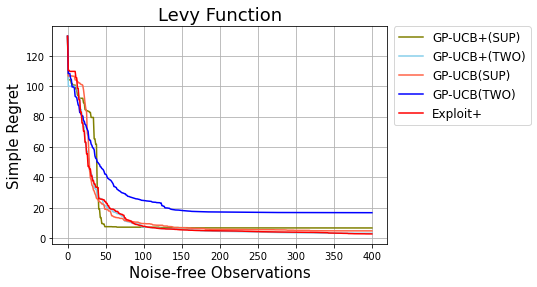}
\caption{Simple regret plots for benchmark functions with $\beta_t^{1/2} = 2$ (TWO) and  $\beta_t^{1/2} = \max_{x\in \mathcal{X}_D } |f(x)|$ (SUP).}
  \label{Benchmark_beta_CUM_PLOT}
\end{figure}

\section{Additional Figures and Implementation Details: Hyperparameter Tuning}\label{appendix:RF}
To train the random forest regression model for California housing dataset \cite{pace1997sparse}, we first divided the dataset into test and train datasets. 80 percent of (feature vector, response) pairs were assigned to be the training set, while the remaining 20 percent were treated as a test set. In constructing the deterministic objective function, we defined it to be a mapping from the vector of four hyperparameters to a negative test error of the model built based on the input and training set. As the model construction may involve randomness coming from the bootstrapped samples, we fixed the random state parameter to remove any such randomness in the definition of the objective. We tuned the following four hyperparameters:
\begin{itemize}
    \item Number of trees in the forest $\in [10, 200].$
    \item Maximum depth of the tree $\in [1, 20].$
    \item Minimum  number of samples requires to split the internal node $\in [2, 10].$
    \item Maximum proportion of the number of features to consider when looking for the best split $\in [0.1, 0.999].$
\end{itemize}
For the first three parameters we conducted the optimization task in the continuous domain and rounded down to the nearest integers. Figure \ref{EXAMPLE_RF_APPEND_CAL_HOUSE} shows that the proposed algorithms attained smaller cumulative and instantaneous test errors. 

\begin{figure}[H]
    \centering
\includegraphics[height=4cm]{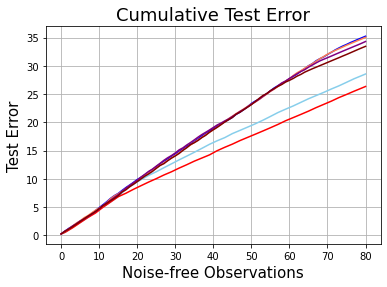}
\includegraphics[height=4cm]{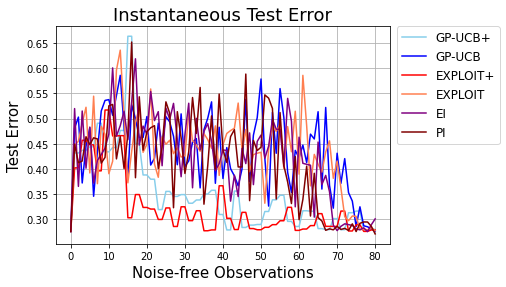}
  \caption{Test errors vs number of noise-free observations.}
    \label{EXAMPLE_RF_APPEND_CAL_HOUSE}
\end{figure}

\section{Implementation Details: Garden Sprinkler Computer Model}\label{appendix:GARDEN}
For the Garden Sprinkler computer model, the eight-dimensional search space we considered was given by:
\begin{itemize}
    \item Vertical nozzle angle $\in [0, 90].$ 
    \item Tangential nozzle angle $\in [0, 90].$ 
    \item Nozzle profile $\in [2 \times 10^{-6}, 4 \times 10^{-6}].$ 
    \item Diameter of the sprinkler head $\in [0.1, 0.2].$ 
    \item Dynamic friction moment $\in [0.01, 0.02].$ 
    \item Static friction moment $\in [0.01, 0.02].$ 
    \item Entrance pressure $\in [1, 2].$ 
    \item Diameter flow line $\in [5, 10].$ 
\end{itemize}

\section{Implementation Details: Bayesian Inference for Parameters of Differential Equations}\label{appendix:BI}
The number of initial design points, the step-size parameter for the random walk Metropolis-Hastings algorithm, and the choice of numerical ODE solver were given by:
\begin{itemize}
    \item Total number of function evaluations: 20 / Size of initial design points: 2 (Rossler)
    \item Total number of function evaluations: 400 / Size of initial design points: 20 (Lorenz 63)
    \item Perturbation kernel for the random walk Metropolis-Hastings: $\mathcal{N}(0, 0.3 I_{3\times3})$ (Lorenz 63). \\ This choice roughly leads to an acceptance rate of 50 percent.
    \item ODE solvers: RK45 implemented through $\texttt{solve\_ivp}$ function in Python's scipy package. (Both Rossler and Lorenz 63.)
\end{itemize}

\end{document}